\documentclass[11pt]{article}

\usepackage[margin=1in]{geometry}
\usepackage[numbers,sort&compress]{natbib}
\usepackage{amsmath,amsthm,amssymb,mathtools}
\usepackage{graphicx}
\usepackage{hyperref}
\hypersetup{colorlinks=true, linkcolor=blue, citecolor=blue, urlcolor=blue}
\usepackage{microtype}
\usepackage{booktabs}
\usepackage{enumitem}
\setlist{nosep,leftmargin=*}

\newtheorem{theorem}{Theorem}[section]
\newtheorem{lemma}[theorem]{Lemma}
\newtheorem{proposition}[theorem]{Proposition}
\newtheorem{corollary}[theorem]{Corollary}
\theoremstyle{definition}
\newtheorem{assumption}[theorem]{Assumption}
\theoremstyle{remark}
\newtheorem{remark}[theorem]{Remark}

\DeclareMathOperator*{\E}{\mathbb{E}}
\newcommand{\R}{\mathbb{R}}
\newcommand{\ind}{\mathbb{1}}
\newcommand{\Otil}{\widetilde{O}}

\title{Bilateral Trade Under Heavy-Tailed Valuations:\\Minimax Regret without a Variance Bound}
\author{Hangyi Zhao \\ \texttt{hyz0815@stanford.edu}}

\begin{document}
\maketitle

\begin{abstract}
In contextual bilateral trade under full feedback, the posted price does not affect which valuations are observed. We show that in this model such \emph{action-independent} feedback removes the polynomial adaptation penalty familiar from heavy-tailed bandits: fully parameter-free algorithms attain the oracle minimax $T$-exponents up to logarithmic factors, with no knowledge of the moment order $p \in (1,2)$ or its scale $\sigma_p$, and --- in the nonparametric case --- none of the effective H\"older smoothness $\beta \in (0,1]$. The statistic that makes model selection possible is a paired squared-loss difference, whose noise-square term cancels exactly, leaving noise damped by the candidate gap.
The resulting bilateral-trade regret rates are new. Trader valuations have bounded conditional densities and heavy tails --- finite $p$-th moments for some $p \in (1,2)$, with possibly infinite variance. An epoch-based algorithm with truncated means achieves regret $\widetilde{O}(T^{(2-p)/p})$ in the parametric model and $\widetilde{O}(T^{1-2\beta(p-1)/(\beta p + d(p-1))})$ when the market value function is $\beta$-H\"older, with matching $\Omega(\cdot)$ lower bounds --- under a mild nondegeneracy condition --- via Assouad's method and a fixed-support mixture construction --- characterizing the minimax rate in $T$ up to logarithmic factors over the effective smoothness range $\beta \in (0,1]$, interpolating between the classical nonparametric rate at $p{=}2$ and the trivial linear rate as $p \to 1^+$.
The enabling structural step extends the self-bounding property of Bachoc et al.\ (ICML 2025) from bounded to real-valued valuations: within our conditionally independent, conditionally centered noise model, bounded conditional densities and finite first moments suffice for the expected regret of any price $\pi$ to satisfy $\mathbb{E}[g(m,V,W) - g(\pi,V,W)] \le L|m-\pi|^2$ --- no second moment is needed.
\end{abstract}

\section{Introduction}\label{sec:intro}

Bilateral trade---the simplest two-sided market---requires a broker to set prices between two traders whose private valuations are unknown --- the lower-valuation trader selling to the higher-valuation one whenever the posted price separates them.
The celebrated impossibility theorem of Myerson and Satterthwaite~\cite{myerson1983efficient} shows that no incentive-compatible, individually rational, budget-balanced mechanism can achieve full efficiency in a single round.
This spurred a rich literature on approximate mechanisms~\cite{blumrosen2016almost} and, more recently, on \emph{online} bilateral trade, where the broker learns from repeated interactions and regret---the cumulative loss from suboptimal pricing---replaces the single-shot efficiency objective.

Bachoc, Cesari, and Colomboni~\cite{bachoc2025parametric} initiated the study of \emph{contextual} online bilateral trade, where a public context vector $x_t$ arrives each round and trader valuations depend on an unknown function $m(x_t)$.
Under bounded noise densities and finite variance, they established an $O(Ld\log T)$ parametric regret bound and an $\widetilde{O}(\sqrt{LdT})$ bound under two-bit feedback; the nonparametric setting was subsequently treated in~\cite{bachoc2025nonparametric}.
A key structural insight is the \emph{self-bounding property}: the expected regret of pricing at $\pi$ instead of $m$ is at most $L|m-\pi|^2$, reducing regret control to mean estimation.
However, their algorithms rely on ordinary least squares, which requires finite variance ($\mathbb{E}[\xi^2] < \infty$).
In many applications---financial markets, insurance, real estate---valuations exhibit heavy tails well-modeled by Student's $t(\nu)$ with $1 < \nu < 2$, where the variance is infinite~\cite{lugosi2019mean}.
This raises a natural question: \emph{what regret is achievable when the noise has bounded density but only finite $p$-th moments for some $p < 2$ --- so possibly infinite variance?}

\emph{When does a learner pay for not knowing the difficulty of its problem?}
In heavy-tailed bandits the learner pays through exploration: adaptation to an unknown moment bound faces a proven frontier between distribution-dependent and distribution-free guarantees~\cite{ashutosh2021bandit,genalti2024adaptive}, and parameter-free bandit algorithms, which do exist, reach that frontier by scheduling exploration against a tail they do not know~\cite{genalti2026parameterfree}.
This paper isolates the role of the feedback interface. Under full feedback the observations do not depend on the posted price, and a greedy algorithm built on median-of-means is parameter-free at the oracle rate: there is no exploration schedule to calibrate, so estimation-side adaptivity costs nothing beyond logarithmic factors.
Model selection remains --- the candidate cell width is still unknown --- and it too is free, but only when candidates are compared by the right statistic.
Realized-gain comparisons are the wrong one: the gain difference of two nearby prices is a rare event carrying a $\Theta(1)$ payoff, so its $p$-th moment shrinks only \emph{linearly} in the candidate gap (Remark~\ref{rem:whyvalidation} records the resulting heuristic calibration).
Paired validation-loss differences are the right one: against the observable pseudo-label $Y = (V+W)/2$, writing $\eta$ for its centered noise and $\Delta := \hat m_h - \hat m_{h'}$ for the candidate gap,
\[
  \bigl(\hat m_h(x) - Y\bigr)^2 - \bigl(\hat m_{h'}(x) - Y\bigr)^2
  = \bigl[(\hat m_h - m)^2 - (\hat m_{h'} - m)^2\bigr](x) \;-\; 2\,\eta\,\Delta(x),
\]
so the noise's square cancels exactly, the noise enters damped by the candidate gap, and oracle-rate selection follows with fewer assumptions (Section~\ref{sec:adaptive:np}).
In this model the price of adaptivity is the price of exploration plus the price of comparing the wrong statistic; under full feedback, with the right statistic, no polynomial penalty remains.

We answer these questions with four contributions. Contributions (C1)--(C3) establish the oracle minimax benchmark; (C4) is the conceptual payoff: under full feedback that benchmark is attained with no knowledge of any class parameter, against the polynomial adaptation frontier of heavy-tailed bandits.
\textbf{(C1)} We extend the self-bounding property from bounded to real-valued valuations (Lemma~\ref{lem:selfbound}), showing that, within the conditionally independent, conditionally centered noise model of Assumption~\ref{ass:density}, bounded conditional densities and finite first moments ($\E[|\xi_t|], \E[|\zeta_t|] < \infty$ --- no second or even $p$-th moment) suffice to control regret via squared estimation error.
\textbf{(C2)} We design epoch-based algorithms using truncated-mean estimation~\cite{bubeck2013bandits} and prove regret rates that are tight over the effective range $\beta \in (0,1]$: $\widetilde{O}(T^{(2-p)/p})$ in the parametric case and $\widetilde{O}(T^{1 - 2\beta(p-1)/(\beta p + d(p-1))})$ in the nonparametric case, where $p \in (1,2)$ is the moment parameter and $\beta$ is the H\"older smoothness.
\textbf{(C3)} We establish matching lower bounds via Assouad's method~\cite{tsybakov2009introduction} combined with a fixed-support mixture construction (Propositions~\ref{prop:lower} and~\ref{prop:paramlower}), proving that our rates are minimax optimal in $T$ up to logarithmic factors, over the effective range $\beta \in (0,1]$ and for nondegenerate class parameters (the parametric dimension dependence of the lower bound is left open; Remark~\ref{rem:paramlowerd}).
The constructed instances also satisfy Assumption~\ref{ass:regular}, a mild regularity condition at zero under which the optimal price is unique.
\textbf{(C4)} We show that tail-adaptivity is \emph{free} under full feedback (Section~\ref{sec:adaptive}): replacing truncation by median-of-means yields a fully parameter-free algorithm achieving the parametric oracle rate simultaneously for all $p \in (1,2)$ (Theorem~\ref{thm:adaptive}), and a validation-based width selection extends this jointly to $(p, \beta)$ in the nonparametric model, for all smoothness levels and with no additional assumption (Theorem~\ref{thm:adaptivenp}).
Under bandit feedback such adaptivity is provably costly~\cite{ashutosh2021bandit,genalti2024adaptive}: the price of adaptivity is a price of exploration, not of estimation.

\subsection{Related work}\label{sec:related}

\textbf{Bilateral trade.}
Online bilateral trade was initiated by Cesa-Bianchi et al.~\cite{cesabianchi2024bilateral} and has since been studied under various feedback models and distributional assumptions; see~\cite{bachoc2025parametric} and references therein.
The contextual setting was introduced in~\cite{bachoc2025parametric} (parametric) and~\cite{bachoc2025nonparametric} (nonparametric), both under finite-variance noise.
Our work extends this line to the infinite-variance regime.

\textbf{Robust mean estimation.}
Estimation under heavy tails has a long history.
Catoni~\cite{catoni2012challenging} introduced influence-function estimators achieving sub-Gaussian deviations under finite variance.
Bubeck, Cesa-Bianchi, and Lugosi~\cite{bubeck2013bandits} analyzed truncated-mean estimators under finite $p$-th moments ($p \in (1,2]$), achieving the minimax rate $(u/n)^{(p-1)/p}$; this is the estimator we use.
Lugosi and Mendelson~\cite{lugosi2019mean} survey the broader landscape, including median-of-means and multivariate extensions; Devroye et al.~\cite{devroye2016subgaussian} treat sub-{G}aussian mean estimation under minimal assumptions, including the price of adapting to unknown moment parameters.
Hopkins~\cite{hopkins2020mean} showed that sub-Gaussian rates are achievable in polynomial time under finite variance; our regime ($p < 2$, no uniform variance bound) lies strictly below this threshold.

\textbf{Online learning with heavy tails.}
Bubeck et al.~\cite{bubeck2013bandits} studied multi-armed bandits under heavy-tailed rewards, and Medina and Yang~\cite{medina2016noregret} extended this to linear bandits.
Our problem differs from standard bandits: the self-bounding property of bilateral trade \emph{squares} the estimation error in the regret decomposition, altering the rate.
Under bandit feedback, adaptivity to an unknown moment bound faces a proven frontier, which parameter-free bandit algorithms meet by scheduled exploration~\cite{ashutosh2021bandit,genalti2024adaptive,genalti2026parameterfree}; we show that full feedback removes the polynomial part of this barrier in our setting (Section~\ref{sec:adaptive}).

\textbf{Model selection under heavy tails.}
The algebra behind our width selector --- paired squared-loss differences, in which the noise's square cancels and the noise enters multiplied by the candidate gap --- is classical: it is the multiplier decomposition underlying hold-out estimates and aggregation, and the selector is a median-of-means tournament in the sense of Lugosi and Mendelson~\cite{lugosi2016tournaments}.
What is specific to our setting is \emph{which selection interfaces survive heavy tails}: Lepski-type confidence-interval selection would need an observable radius that upper-bounds a heavy-tailed mean, unavailable without $(p, \sigma_p)$ --- a mechanism-level obstacle we record as a heuristic, not a proven impossibility (Remark~\ref{rem:whyvalidation}); realized-gain comparisons have a gap-localized $p$-th moment shrinking only linearly in the gap (Lemma~\ref{lem:paired}(iii), proved in Appendix~\ref{app:adaptive}); paired validation losses inherit neither defect, and because their expectation is the risk gap itself, no regularity of the noise at zero is needed (Theorem~\ref{thm:adaptivenp}).
The reduction to mean estimation via Lemma~\ref{lem:selfbound} is new for unbounded valuations, and it does not trivialize selection: full feedback supplies the pseudo-labels, but only the validation statistic uses them at oracle rate.

\textbf{Nonparametric regression, and what we do not claim.}
The minimax rate $T^{d/(2\beta+d)}$ for $\beta$-H\"older regression in $d$ dimensions was established by Stone~\cite{stone1982optimal} under finite variance, and regression under weak moment assumptions is by now a developed subject: Brownlees, Joly and Lugosi~\cite{brownlees2015empirical} analyze empirical risk minimization with heavy-tailed losses via robust mean estimates, Lugosi and Mendelson~\cite{lugosi2016tournaments} obtain optimal accuracy--confidence trade-offs by median-of-means tournaments, and~\cite{lugosi2019mean} surveys the area.
Heavy-tailed nonparametric regression itself has a developed literature: Brown, Cai and Zhou~\cite{brown2008robust} achieve adaptive rates by wavelet \emph{median} regression over Besov classes (no moment condition, but a median rather than a mean target); Han and Wellner~\cite{han2019convergence} characterize exactly when least squares retains Gaussian-equivalent rates under weak moments; and Shen, Jiao, Lin and Huang~\cite{shen2021robust} prove excess-risk bounds for robust deep-network regression assuming only $p > 1$ moments. All of these concern offline estimation risk; none proves regret bounds, lower bounds constrained to a pricing model's own class, or anything about feedback interfaces.
We claim no novelty for these tools. Truncated means, median-of-means, doubling epochs, Assouad's method, and selection by paired squared losses are all standard, and we use them as such.
What is new is the object they are applied to and the conclusions that follow: the \emph{regret} of a bilateral-trade broker rather than an estimation error --- a quantity that Lemma~\ref{lem:selfbound} relates to squared pricing error for \emph{unbounded} valuations under finite first moments alone; a lower bound constrained to hold within the trade model's own class, where bounded conditional densities and conditional centering forbid the usual location-shift constructions (Section~\ref{sec:lower}, and Q5 of Section~\ref{sec:discussion}); and the resulting statement about the feedback interface (Corollary~\ref{cor:freeadapt} and Remark~\ref{rem:interfacesep}), which concerns neither estimation rates nor tails as such.

\section{Setup and the Structural-Algorithmic Gap}\label{sec:setup}

We study online contextual bilateral trade where noise may have infinite variance.
Over $T$ rounds, at round $t$: a context $x_t \in \mathcal{X} \subset \R^d$ is revealed ($\mathcal{X} = [0,1]^d$ throughout the nonparametric development; the parametric results, Theorems~\ref{thm:parametric} and~\ref{thm:adaptive}, instead normalize contexts to the Euclidean unit ball, as their statements make explicit); two traders arrive with valuations $V_t = m(x_t) + \xi_t$ and $W_t = m(x_t) + \zeta_t$; the broker posts price $P_t$ and observes $(V_t, W_t)$ (full feedback). The broker is \emph{nonanticipating}: $P_t$ may depend on the history $(x_s, V_s, W_s)_{s<t}$, the current context $x_t$, and internal randomization, but not on the current $(V_t, W_t)$; all algorithms below, and the algorithm class quantified over in the lower bounds, are of this form. The gain from trade is $g(p,v,w) = (v \vee w - v \wedge w)\cdot\ind\{v \wedge w \le p \le v \vee w\}$, and regret is $R_T = \sum_{t=1}^T \E[g(m(x_t), V_t, W_t) - g(P_t, V_t, W_t)]$. The gain is the symmetric (\emph{brokerage}) form used by the contextual line we extend~\cite{bachoc2025parametric}: the traders' roles --- buyer and seller --- are determined ex post by the realized valuations, and when the seller's valuation is a.s.\ the lower one it coincides with the classical directed gain $(w - v)\,\ind\{v \le p \le w\}$.

\begin{assumption}[Conditional Bounded Density]\label{ass:density}
The triples $(x_t,\xi_t,\zeta_t)$ are i.i.d.\ across $t$. Conditionally on $x_t$,
the noises $\xi_t,\zeta_t$ are independent with conditional densities
$f_\xi(\cdot\mid x_t)$, $f_\zeta(\cdot\mid x_t)$ bounded by $L\ge 1$, and are
conditionally centered: $\E[\xi_t\mid x_t]=\E[\zeta_t\mid x_t]=0$.
This is \emph{necessary}: already in the context-independent subclass,
unbounded densities force $R_T=\Omega(T)$~\cite{bachoc2025parametric}.
\end{assumption}

\begin{assumption}[Conditional Finite $p$-th Moment]\label{ass:moment}
$\E[|\xi_t|^p\mid x_t]\le\sigma_p^p$ and $\E[|\zeta_t|^p\mid x_t]\le\sigma_p^p$
almost surely, for some $p\in(1,2]$; for $p < 2$ the conditional variances may be infinite.
The \emph{context-independent} model, $f_\xi(\cdot\mid x)\equiv f_\xi$, is the
special case in which the conditional law does not depend on $x$; the model
of~\cite{bachoc2025parametric} is in turn its bounded-valuation special case.
\end{assumption}

\begin{assumption}[Smoothness]\label{ass:smooth}
$m:[0,1]^d \to \R$ is $(\beta, L_H)$-H\"older: $|m(x) - m(x')| \le L_H \|x - x'\|^\beta$.
\end{assumption}

\noindent For $\beta > 1$ this increment condition forces $m$ to be constant (chaining), so the effective smoothness range of Assumption~\ref{ass:smooth} is $\beta \in (0, 1]$; the higher-order H\"older scale ($k$-th derivatives $\alpha$-H\"older) requires local-polynomial estimators and is left to future work.

\begin{assumption}[Regularity at zero]\label{ass:regular}
For a.e.\ $x$, $f_\xi(\cdot\mid x)$ and $f_\zeta(\cdot\mid x)$ are continuous at $0$,
and $\operatorname*{ess\,inf}_x\,[f_\xi(0\mid x)+f_\zeta(0\mid x)]>0$.
\end{assumption}

\noindent Assumption~\ref{ass:regular} is used only for the unique-maximizer claim of Lemma~\ref{lem:selfbound}; the self-bounding inequality~\eqref{eq:selfbound} and the upper bounds of Theorems~\ref{thm:parametric}--\ref{thm:nonparametric} hold without it, and the lower-bound instances of Section~\ref{sec:lower} satisfy it by construction.
It is mild: if the combined conditional noise density vanishes on an interval around the origin, every price in a neighborhood of $m(x)$ attains the same conditional expected gain, so the optimal price is not identifiable and uniqueness genuinely fails (Appendix~\ref{app:selfbound}).

\noindent\textbf{The gap.}
Bachoc et al.~\cite{bachoc2025parametric} proved the self-bounding property
$\E[g(m, V, W) - g(\pi, V, W)] \le L|m - \pi|^2$
for valuations in $[0,1]$: in our terms, the context-independent special case of Assumptions~\ref{ass:density}--\ref{ass:moment} with bounded noise, where variance is automatically finite---infinite variance is impossible in that model.
Their Algorithm~2 further uses OLS on $Y_t = (V_t{+}W_t)/2$, requiring $\E[\xi^2] < \infty$.
We make two contributions: \textbf{(C1)}~extend the self-bounding property to $V_t,W_t \in \R$ (Lemma~\ref{lem:selfbound}), and \textbf{(C2)}~close the algorithmic gap via robust estimation under finite $p$-th moment (Theorems~\ref{thm:parametric}--\ref{thm:nonparametric}).

\begin{remark}[Two-bit feedback]\label{rem:twobit}
Under two-bit feedback, observations are binary---bounded regardless of noise tails. The rate $\Otil(\sqrt{LdT})$ from~\cite{bachoc2025parametric} holds without finite variance. Our contribution is specific to full feedback.
\end{remark}

\noindent Economically, full feedback is a post-trade transparency regime: both reservation values become observable after the round --- as when transaction records, quotes, or appraisals reveal what each side would have accepted --- whereas two-bit feedback is the minimal information a posted-price mechanism generates by itself. Our question is what the learner gains when the market reveals more than the mechanism does.

\section{Main Results}\label{sec:results}

\begin{lemma}[Generalized self-bounding property]\label{lem:selfbound}
Under Assumption~\ref{ass:density} with $\E[|\xi_t|],\E[|\zeta_t|] < \infty$ (implied by $p > 1$), for all $\pi \in \R$:
\begin{equation}\label{eq:selfbound}
  \E[g(m, V, W) - g(\pi, V, W)] \le L\,|m - \pi|^2.
\end{equation}
If in addition Assumption~\ref{ass:regular} holds, then $m$ is the \emph{unique} maximizer of $\pi \mapsto \E[g(\pi,V,W)]$.
Under Assumption~\ref{ass:moment}, $\E[g(m,V,W)] \le 2\sigma_p$.
Under Assumption~\ref{ass:density}, the lemma applies conditionally on $x_t$ (with $m = m(x_t)$ and the conditional densities), and the constants $L$ and $2\sigma_p$ are uniform in the context.
\end{lemma}

\begin{proof}[Proof sketch (full proof in Appendix~\ref{app:selfbound})]
Write $\delta = \pi - m$. Define $h(\delta) = \E[g(m{+}\delta, V, W)]$.
By independence of $\xi,\zeta$ and dominated convergence (the integrand is bounded by $|\xi{-}\zeta|$, integrable since $p > 1$), one obtains, for almost every $\delta$ (the densities are merely bounded, so $h$ is differentiable only a.e.):
\begin{equation}\label{eq:derivative}
  h'(\delta) = -\delta\bigl[f_\xi(\delta) + f_\zeta(\delta)\bigr].
\end{equation}
Integrating (valid since $h$ is locally Lipschitz; Appendix~\ref{app:selfbound}) and using $f_\xi,f_\zeta \le L$:
$h(0) - h(\delta) = \int_0^{|\delta|} s\,[f_\xi(\pm s) + f_\zeta(\pm s)]\,ds \le L|\delta|^2$.
Under Assumption~\ref{ass:regular} the integrand is bounded below by $c\,s > 0$ on a neighborhood of $0$, so $h(0) - h(\delta) > 0$ for every $\delta \ne 0$, giving uniqueness.\qed
\end{proof}

\begin{theorem}[Parametric]\label{thm:parametric}
Let $m(x) = x^\top \phi$, $\phi \in \R^d$ with $\|\phi\| \le B$ for a known bound $B > 0$. Under Assumptions~\ref{ass:density}--\ref{ass:moment} with full feedback, assuming $x_t$ are i.i.d.\ with $\|x_t\| \le 1$ a.s.\ and $\Sigma := \E[x_t x_t^\top] \succeq \lambda I_d$ for some $\lambda > 0$: the epoch algorithm of Section~\ref{sec:proof:param}, which uses knowledge of $(p, \sigma_p, B)$ through the truncation level~\eqref{eq:muhat}, achieves
\[
  R_T = \Otil\!\left(L\, d\, (B + \sigma_p)^2\, \lambda^{-2}\, T^{(2-p)/p}\right),
\]
where the $\Otil$ absorbs polylogarithmic factors and absolute constants only (the proof's leading coefficient is $196\,L\,d\,(B+\sigma_p)^2/\lambda^2$).
As $p \to 2^-$ the exponent vanishes, consistent with the polylogarithmic-regret regime of~\cite{bachoc2025parametric} at $p{=}2$ (whose exact $O(Ld\log T)$ our epoch analysis does not itself recover). As $p \to 1^+$: approaches $\Otil(T)$.
\end{theorem}

\begin{theorem}[Nonparametric]\label{thm:nonparametric}
Under Assumptions~\ref{ass:density}--\ref{ass:smooth}, with $x_t$ having density $\ge \mu_0 > 0$ on $[0,1]^d$: the cell-partition epoch algorithm of Section~\ref{sec:proof:nonparam}, which uses knowledge of $(p, \sigma_p, \beta, L_H)$ and a bound $B_m \ge \sup_x|m(x)|$ through the cell width and the truncation level, achieves
$R_T = \Otil\!\left(T^{\;1 - 2\beta(p-1)/(\beta p + d(p-1))}\right)$.
When $p{=}2$: gives $\Otil(T^{d/(2\beta+d)})$, the classical rate~\cite{stone1982optimal}. As $p \to 1^+$: gives $\Otil(T)$.
\end{theorem}

\begin{remark}[Rate comparison]\label{rem:rates}
{\small
\begin{center}
\begin{tabular}{@{}lcc@{}}
\toprule
Setting & Variance & Regret \\
\midrule
Parametric, $p{=}2$ & finite & $O(Ld\log T)$ \cite{bachoc2025parametric} \\
Parametric, $p{\in}(1,2)$ & possibly $\infty$ & $\Otil(T^{(2-p)/p})$ [Thm~\ref{thm:parametric}] \\
Nonparametric, $p{=}2$ & finite & $\Otil(T^{d/(2\beta+d)})$ \cite{bachoc2025nonparametric} \\
Nonparametric, $p{\in}(1,2)$ & possibly $\infty$ & $\Otil(T^{1-2\beta(p{-}1)/(\beta p{+}d(p{-}1))})$ [Thm~\ref{thm:nonparametric}] \\
\bottomrule
\end{tabular}
\end{center}}
\end{remark}

Our third and fourth main results remove all parameter knowledge from these guarantees; the mechanisms --- median-of-means greedy pricing and validation-loss width selection --- are developed in Section~\ref{sec:adaptive}, with proofs in Appendix~\ref{app:adaptive}.

\begin{theorem}[Tail-adaptive parametric pricing]\label{thm:adaptive}
Under the conditions of Theorem~\ref{thm:parametric} --- but with none of $(p, \sigma_p, B)$ known to the algorithm --- the median-of-means greedy algorithm of Section~\ref{sec:adaptive:param} (MoM-greedy), whose inputs are $d$ and $T$ alone, satisfies, simultaneously for every $p \in (1,2)$, $\sigma_p > 0$, $B > 0$ and $\lambda > 0$:
\[
  R_T = \Otil\!\left(L\, d\, (B + \sigma_p)^2\, \lambda^{-2}\, T^{(2-p)/p}\right),
\]
the oracle rate of Theorem~\ref{thm:parametric}, with the $\Otil$ absorbing the same polylogarithmic factors and absolute constants only. Assumption~\ref{ass:regular} is not used.
\end{theorem}

\begin{corollary}[Variance-blind polylogarithmic regret at $p = 2$]\label{rem:p2}
Lemma~\ref{lem:mom} covers $p \in (1,2]$ and the epoch bound~\eqref{eq:geosum} is valid for every $\alpha \in (0,1]$, so Theorem~\ref{thm:adaptive} extends to $p = 2$ with no change to the assembly.
In particular, under the conditions of Theorem~\ref{thm:parametric} with Assumption~\ref{ass:moment} holding at $p = 2$ --- conditional variances bounded by $\sigma_2^2$, the bound unknown to the algorithm --- MoM-greedy, with inputs $(d, T)$ alone, achieves $R_T = \Otil\bigl(L\,d\,(B+\sigma_2)^2\,\lambda^{-2}\bigr)$: regret polylogarithmic in $T$, variance-blind.
(The polylogarithmic exponent is not optimized; cf.\ the $O(Ld\log T)$ of~\cite{bachoc2025parametric} under known variance and bounded valuations.)
\end{corollary}

\begin{theorem}[Tail- and smoothness-adaptive nonparametric pricing]\label{thm:adaptivenp}
Under Assumptions~\ref{ass:density}--\ref{ass:smooth}, with $x_t$ having density $\ge \mu_0 > 0$ on $[0,1]^d$, there is an algorithm with inputs $(d, T)$ alone --- median-of-means validation selection over the cell-width grid $\{2^{-i}\}_{i \le \lceil \log_2 T\rceil}$, described in Section~\ref{sec:adaptive:np} and Appendix~\ref{app:adaptive} --- achieving, simultaneously for all $p \in (1,2)$ and all $(\beta, L_H, \sigma_p, B_m)$ with $\beta > 0$:
\[
  R_T = \Otil\!\left(T^{\;1 - 2\beta(p-1)/(\beta p + d(p-1))}\right),
\]
the oracle rate of Theorem~\ref{thm:nonparametric}. No regularity of the noise at zero is assumed.
\end{theorem}

\begin{corollary}[No polynomial price of adaptivity under full feedback]\label{cor:freeadapt}
Fix any $p \in (1,2)$ and $\sigma_p > 0$, together with the remaining class parameters: $(B, \lambda)$ in the parametric model of Theorem~\ref{thm:parametric}, or $(\beta, L_H, B_m, \mu_0)$ with $\beta > 0$ in the nonparametric model of Theorem~\ref{thm:nonparametric}.
\begin{enumerate}
\item[(i)] The parameter-free algorithms of Theorems~\ref{thm:adaptive} and~\ref{thm:adaptivenp}, whose inputs are $(d, T)$ alone, attain the regret bounds of the corresponding known-parameter algorithms of Theorems~\ref{thm:parametric} and~\ref{thm:nonparametric} up to factors polylogarithmic in $T$. Hence the price of not knowing $(p, \sigma_p, B)$ --- respectively $(p, \sigma_p, \beta, L_H, B_m)$ --- contains no factor $T^{c}$ with $c > 0$. In the parametric model this holds at $p = 2$ as well (Corollary~\ref{rem:p2}).
\item[(ii)] If in addition $(2L\sigma_p)^p \ge 8$, and $\beta \in (0,1]$ in the nonparametric case, then those bounds are minimax-optimal in $T$ up to logarithmic factors (Propositions~\ref{prop:lower} and~\ref{prop:paramlower}), so the parameter-free algorithms attain the minimax $T$-exponent itself.
\end{enumerate}
Part~(ii) inherits exactly the hypotheses of the lower bounds, including their scope: tightness is asserted in $T$ only, and the parametric dimension dependence is left open (Remark~\ref{rem:paramlowerd}).
\end{corollary}

\begin{remark}[The contrast with bandit feedback]\label{rem:interfacesep}
Under heavy-tailed bandit feedback the corresponding price is provably polynomial: adaptation to an unknown moment bound faces a sharp frontier between distribution-dependent and distribution-free regret, which parameter-free bandit algorithms meet by scheduling exploration against the tail they do not know~\cite{ashutosh2021bandit,genalti2024adaptive,genalti2026parameterfree}.
The comparison is across models --- we prove the full-feedback side; the bandit side is the cited literature --- and the mechanism difference is what Section~\ref{sec:adaptive} isolates: here the data do not depend on the posted price, so no exploration schedule exists to mis-calibrate, and the remaining adaptation problem, model selection, is solved at oracle rate by the paired validation statistic.
\end{remark}

\section{Proof of Theorem~\ref{thm:parametric}}\label{sec:proof:param}

The proof has four steps.
We first construct a robust estimator of $\phi$ using coordinate-wise truncated means of score vectors $S_s = x_s Y_s$, obtaining high-probability error bounds under finite $p$-th moments alone (Step~1).
We then translate this into a prediction error bound via the empirical Gram matrix (Step~2), decompose the per-epoch regret into good and bad events using the self-bounding property of Lemma~\ref{lem:selfbound} (Step~3), and sum a geometric series over doubling epochs to obtain the final rate (Step~4).

\smallskip\noindent\textbf{Algorithm.}
Divide the $T$ rounds into $K = \lceil\log_2 (T{+}1)\rceil$ epochs, where epoch $k$ spans rounds $[2^{k-1}, 2^k) \cap [1, T]$ (the final epoch is truncated at round $T$, so every round is covered) and has length $n_k \le 2^{k-1}$.
In epoch $1$, play the fixed price $0$ (per-round regret at most $2\sigma_p$ by Lemma~\ref{lem:selfbound}).
For $k \ge 2$, use the $n = n_{k-1} = 2^{k-2}$ samples from epoch $k{-}1$ to construct an estimate $\hat\phi_k \in \R^d$ (pseudoinverse in place of $\hat\Sigma^{-1}$ whenever $\hat\Sigma$ is singular; all such early epochs fall inside the burn-in charge of Step~4), then play $P_t = x_t^\top \hat\phi_k$ for all $t$ in epoch $k$.

\smallskip\noindent\textbf{Step 1: Construction of $\hat\phi_k$ via truncated score vectors.}
Let $Y_s = (V_s{+}W_s)/2 = x_s^\top\phi + \eta_s$ where $\eta_s := (\xi_s{+}\zeta_s)/2$.
By Minkowski's inequality applied conditionally on $x_s$ (Assumption~\ref{ass:moment}), $\E[|\eta_s|^p \mid x_s] \le \sigma_p^p$ a.s., hence $\|\eta_s\|_p \le \sigma_p$ by the tower property; moreover $\E[\eta_s \mid x_s] = 0$ by conditional centering.
Form the \emph{score vectors} $S_s = x_s Y_s \in \R^d$, so that $\E[S_s] = \E[x_s x_s^\top \phi] + \E[x_s\,\E[\eta_s \mid x_s]] = \Sigma\phi$.

For each coordinate $j \in [d]$, the $j$-th component satisfies $[S_s]_j = [x_s]_j\,Y_s$. Since $|[x_s]_j| \le \|x_s\| \le 1$, conditionally on $x_s$ and then by the tower property:
\begin{equation}\label{eq:rawmoment}
  \E\bigl[\bigl|[S_s]_j\bigr|^p\bigr] \le \E[|Y_s|^p]
  \le \bigl(\|\phi\| + \sigma_p\bigr)^p \le (B + \sigma_p)^p =: u.
\end{equation}
Define the coordinate-wise truncated mean of $\Sigma\phi$:
\begin{equation}\label{eq:muhat}
  \hat\mu_j = \frac{1}{n}\sum_{s=1}^{n} [S_s]_j \cdot \ind\!\bigl\{|[S_s]_j| \le \tau\bigr\}, \quad
  \tau = \Bigl(\frac{u\,n}{\log(dT)}\Bigr)^{\!1/p},\quad j \in [d].
\end{equation}
The variables $\{[S_s]_j\}_{s=1}^n$ are i.i.d.\ with mean $[\Sigma\phi]_j$ and raw moment bound~\eqref{eq:rawmoment}.
By Lemma~\ref{lem:truncated} in Appendix~\ref{app:truncated} (a variant of Lemma~1 of~\cite{bubeck2013bandits}, with $1{+}\varepsilon = p$ in their notation, requiring only independence and a uniform $p$-th moment bound), applied with $\delta = 1/(dT)$: for each $j$, with probability $\ge 1 - 2/(dT)$,
\begin{equation}\label{eq:percoord}
  |\hat\mu_j - [\Sigma\phi]_j| \le 4\,u^{1/p}\!\left(\frac{\log(dT)}{n}\right)^{\!(p-1)/p}
  = 4(B + \sigma_p)\!\left(\frac{\log(dT)}{n}\right)^{\!(p-1)/p}\!.
\end{equation}
A union bound over $j \in [d]$ gives: with probability $\ge 1 - 2/T$,
\begin{equation}\label{eq:muconc}
  \|\hat\mu - \Sigma\phi\|_\infty \le 4(B + \sigma_p)\!\left(\frac{\log(dT)}{n}\right)^{\!(p-1)/p} =: \varepsilon_0.
\end{equation}

Next, form the empirical Gram matrix $\hat\Sigma = n^{-1}\sum_{s=1}^n x_s x_s^\top$.
Since each summand satisfies $\|x_s x_s^\top - \Sigma\|_{\mathrm{op}} \le 1$ and $\E[(x_s x_s^\top - \Sigma)^2] \preceq \Sigma \preceq I_d$, the matrix Bernstein inequality~\cite[Theorem~6.1.1]{tropp2015introduction} gives, with probability $\ge 1 - 1/T$ (assuming WLOG $dT \ge 4$):
\[
  \|\hat\Sigma - \Sigma\|_{\mathrm{op}}
  \le \sqrt{\frac{2\log(2dT)}{n}} + \frac{2\log(2dT)}{3n}
  \le C_1\sqrt{\frac{\log(dT)}{n}}, \qquad C_1 := 3,
\]
where the last step uses $n \ge \log(2dT)$ and $\log(2dT) \le \tfrac32\log(dT)$.
For $n \ge n_0 := 4C_1^2\log(dT)/\lambda^2 = 36\,\lambda^{-2}\log(dT)$ (which also implies $n \ge \log(2dT)$ since $\lambda \le 1$), $\hat\Sigma \succeq (\lambda/2)\,I_d$; the epochs with $n < n_0$ are charged explicitly in Step~4.
Set $\hat\phi_k = \hat\Sigma^{-1}\hat\mu$.

\smallskip\noindent\textbf{Step 2: Prediction error bound.}
Define the ``good event'' $\mathcal{G}_k$: both~\eqref{eq:muconc} and the matrix Bernstein bound hold. Then $\Pr(\mathcal{G}_k) \ge 1 - 3/T$. On $\mathcal{G}_k$:
\begin{align}
  \|\hat\phi_k - \phi\|_2
  &= \|\hat\Sigma^{-1}(\hat\mu - \hat\Sigma\phi)\|_2
  = \|\hat\Sigma^{-1}\bigl[(\hat\mu - \Sigma\phi) + (\Sigma - \hat\Sigma)\phi\bigr]\|_2 \notag\\
  &\le \frac{2}{\lambda}\Bigl(\sqrt{d}\,\varepsilon_0 + C_1\sqrt{\frac{\log(dT)}{n}}\,\|\phi\|\Bigr).\label{eq:phierr}
\end{align}
The first term dominates, non-asymptotically: since $n \ge \log(dT)$ on the epochs considered and $(p{-}1)/p \le 1/2$ for $p \le 2$, we have $(\log(dT)/n)^{1/2} \le (\log(dT)/n)^{(p-1)/p}$, so the Gram term is at most $C_1 B\,(\log(dT)/n)^{(p-1)/p}$.
Combining, with $B \le B + \sigma_p$ and $C_1 = 3 \le 3\sqrt{d}$:
\begin{equation}\label{eq:prederr}
  |P_t - m(x_t)| = |x_t^\top(\hat\phi_k - \phi)| \le \|x_t\| \cdot \|\hat\phi_k - \phi\|_2
  \le \underbrace{\frac{14(B + \sigma_p)\sqrt{d}}{\lambda}}_{=: C_\phi}\left(\frac{\log(dT)}{n}\right)^{\!(p-1)/p}\!.
\end{equation}

\smallskip\noindent\textbf{Step 3: Per-epoch regret.}
Fix epoch $k \ge k_0 := \min\{\kappa \ge 2 : 2^{\kappa-2} \ge n_0\}$, of length $n_k \le 2^{k-1}$ (with equality except possibly for the final epoch, truncated at round $T$), using an estimator built from $n = 2^{k-2} \ge n_0$ samples.

\emph{Good event} ($\mathcal{G}_k$, probability $\ge 1 - 3/T$): By Lemma~\ref{lem:selfbound} and~\eqref{eq:prederr}, the per-round regret is at most $L\,|P_t - m(x_t)|^2 \le L\,C_\phi^2\,(\log(dT)/n)^{2(p-1)/p}$.
Summing over $n_k$ rounds:
\begin{equation}\label{eq:goodepoch}
  R_k^{\mathrm{good}} \le n_k \cdot L\,C_\phi^2\left(\frac{\log(dT)}{n}\right)^{\!2(p-1)/p}
  \le L\,C_\phi^2\,(\log(dT))^{2(p-1)/p}\cdot \frac{2^{k-1}}{(2^{k-2})^{2(p-1)/p}}.
\end{equation}

\emph{Bad event} ($\mathcal{G}_k^c$, probability $\le 3/T$): The per-round regret is at most $\E[g(m, V, W)] \le 2\sigma_p$ by Lemma~\ref{lem:selfbound}. Hence:
\begin{equation}\label{eq:badepoch}
  R_k^{\mathrm{bad}} \le n_k \cdot 2\sigma_p \cdot \Pr(\mathcal{G}_k^c)
  \le 2^{k-1} \cdot 2\sigma_p \cdot \frac{3}{T}.
\end{equation}

\smallskip\noindent\textbf{Step 4: Summing over epochs.}
\emph{Burn-in total.} Epochs $k < k_0$ --- epoch~$1$ (fixed price $0$) and those with $n < n_0$ --- are charged at the cap of Lemma~\ref{lem:selfbound}: per-round regret at most $\E[g(m,V,W)] \le 2\sigma_p$, and $\sum_{k < k_0} n_k \le 2^{k_0-1} \le 4n_0$ by minimality of $k_0$, so their total is at most $8\sigma_p n_0 = 288\,\sigma_p\,\lambda^{-2}\log(dT)$.

\emph{Bad-event total.} $\sum_{k=1}^{K} R_k^{\mathrm{bad}} \le \frac{6\sigma_p}{T}\sum_{k=1}^K 2^{k-1} \le \frac{6\sigma_p}{T}\cdot 2T = 12\sigma_p.$

\emph{Good-event total.} From~\eqref{eq:goodepoch}, writing $\alpha := 2(p{-}1)/p \in (0,1]$:
\begin{equation}\label{eq:epochgeo}
  \sum_{k=k_0}^{K} R_k^{\mathrm{good}}
  \le L\,C_\phi^2\,(\log(dT))^\alpha \sum_{k=2}^{K} \frac{2^{k-1}}{2^{(k-2)\alpha}}
  \le L\,C_\phi^2\,(\log(dT))^\alpha \cdot 2^{1+\alpha}\sum_{k=2}^{K} 2^{k(1-\alpha)}.
\end{equation}
Since $1 - \alpha = (2{-}p)/p \ge 0$, every term is at most the last, and $2^{K} \le 2T$ by the epoch schedule; bounding the sum by its number of terms avoids a $(2^{1-\alpha}-1)^{-1}$ factor, which would blow up as $p \uparrow 2$:
\begin{equation}\label{eq:geosum}
  \sum_{k=2}^{K} 2^{k(1-\alpha)} \;\le\; K\,2^{K(1-\alpha)} \;\le\; K\,(2T)^{1-\alpha} \;\le\; 2K\,T^{(2-p)/p},
\end{equation}
valid for every $\alpha \in (0,1]$, i.e.\ every $p \in (1,2]$; the price is one factor of $K = \lceil\log_2(T{+}1)\rceil$, absorbed by the $\Otil$.
Combining~\eqref{eq:epochgeo}--\eqref{eq:geosum}:
\begin{align*}
  R_T &\le \underbrace{8\sigma_p n_0}_{\text{burn-in}} + \sum_{k=k_0}^K (R_k^{\mathrm{good}} + R_k^{\mathrm{bad}}) \\
  &\le \underbrace{L\,C_\phi^2\,(\log(dT))^{2(p-1)/p}}_{\text{problem constants}} \cdot \underbrace{8K}_{K \,=\, \lceil\log_2(T+1)\rceil} \cdot T^{(2-p)/p} + 12\sigma_p + 288\,\sigma_p\,\lambda^{-2}\log(dT),
\end{align*}
which gives $R_T = \Otil(L\,d\,(B+\sigma_p)^2\lambda^{-2}\,T^{(2-p)/p})$ as claimed, with leading coefficient $L\,C_\phi^2$, $C_\phi^2 = 196\,d\,(B + \sigma_p)^2/\lambda^2$, and only polylogarithmic factors and absolute constants absorbed into $\Otil$ --- in particular no factor degenerating as $p \uparrow 2$. The burn-in term is polylogarithmic and likewise absorbed: for $p < 2$, $T^{(2-p)/p}$ dominates every power of $\log(dT)$, and at $p = 2$ every term in the display is itself polylogarithmic in $T$, the additive $\sigma_p$ terms lying inside the stated $\Otil$ because bounded densities force $\sigma_p^p \ge (2L)^{-p}/(p+1)$ (the bathtub bound of Section~\ref{sec:lower}), hence $\sigma_p \le 4L(B+\sigma_p)^2$. \qed

\section{Proof of Theorem~\ref{thm:nonparametric}}\label{sec:proof:nonparam}

\noindent\textbf{Algorithm (epoch-partition).}
As in Section~\ref{sec:proof:param}, divide the $T$ rounds into $K = \lceil\log_2 (T{+}1)\rceil$ epochs, where epoch $k$ spans rounds $[2^{k-1},2^k) \cap [1,T]$ with length $n_k \le 2^{k-1}$ (final epoch truncated at $T$).
In epoch~$1$, play any fixed price.
For each epoch $k \ge 2$, fix a partition side length $h > 0$ (to be optimized globally), set $q = \lceil h^{-1} \rceil$, and tile $[0,1]^d$ into $M = q^d$ \emph{congruent} axis-aligned cells $\{C_j\}_{j=1}^M$ of side $1/q \in (h/2,\, h]$, with centers $c_j$ (congruence avoids thin boundary cells; all bias bounds below use only the diameter bound $\sqrt{d}\,h$).
Using the $n = n_{k-1} = 2^{k-2}$ samples from epoch $k{-}1$:
for each cell $j$, compute a truncated-mean estimate $\hat{m}_k(c_j)$ of $m(c_j)$ from the observations $\{Y_s : x_s \in C_j\}$ where $Y_s = (V_s{+}W_s)/2$; cells containing no samples are priced at $0$ (the cap $2\sigma_p$ of Lemma~\ref{lem:selfbound} absorbs them, and on the good event of Step~1 no cell is empty past the fill threshold).
In epoch $k$, for each round $t$ with $x_t \in C_j$, play $P_t = \hat{m}_k(c_j)$.

\smallskip\noindent\textbf{Step 1: Cell-count concentration.}
Let $n_j^{(k)} = |\{s \in \text{epoch}\;k{-}1 : x_s \in C_j\}|$.
Since $x_s$ has density $\ge \mu_0$ on $[0,1]^d$ and each cell has volume $q^{-d} \ge (h/2)^d$, $\E[n_j^{(k)}] \ge n\,\mu_0\,(h/2)^d$.
By a multiplicative Chernoff bound, for $n\,\mu_0\,(h/2)^d \ge c_1\log(MT)$:
\begin{equation}\label{eq:chernoff}
  \Pr\!\left(n_j^{(k)} < \tfrac{1}{2}\,\mu_0\,n\,(h/2)^d\right) \le \exp(-\mu_0\,n\,(h/2)^d/8) \le \frac{1}{MT}.
\end{equation}
A union bound over all $M$ cells gives: with probability $\ge 1 - 1/T$, every cell satisfies $n_j^{(k)} \ge \mu_0\,n\,(h/2)^d/2$.
This holds for all epochs $k$ with $n \ge n_0 := C\,h^{-d}\log(h^{-d}T)/\mu_0$; the earlier epochs are charged explicitly in Step~5.

\smallskip\noindent\textbf{Step 2: Per-cell truncated-mean concentration.}
We condition on the design $X^{(k-1)} := (x_s)_{s \in \mathrm{epoch}\,k-1}$, which determines the cell counts $n_j^{(k)}$ and which samples land in each cell; the cell-count event of Step~1 is $X^{(k-1)}$-measurable.
Fix a cell $j$ with $n_j := n_j^{(k)} \ge \mu_0\,n\,(h/2)^d/2$ samples.
Conditionally on $X^{(k-1)}$, the observations $\{Y_s = m(x_s) + \eta_s : x_s \in C_j\}$, $\eta_s = (\xi_s{+}\zeta_s)/2$, are independent (the triples $(x_s,\xi_s,\zeta_s)$ are i.i.d., so given the design the noises are independent with conditional laws determined by $x_s$ and conditionally centered) with \emph{heterogeneous} conditional means $\E[Y_s \mid X^{(k-1)}] = m(x_s)$; their average is the local mean
$\bar{m}_j := n_j^{-1}\sum_{s: x_s \in C_j} m(x_s)$,
which satisfies $|\bar{m}_j - m(c_j)| \le L_H (\sqrt{d}\,h)^\beta$ by the H\"older condition (the cell has diameter $\sqrt{d}\,h$; we use this loose uniform form throughout --- the tighter $(\sqrt{d}\,h/2)^\beta$ changes only constants).
Let $B_m = \sup_x |m(x)|$; this is finite since $m$ is H\"older on the compact domain $[0,1]^d$ ($B_m \le |m(0)| + L_H\,d^{\beta/2}$).
Set $\bar u = (B_m + \sigma_p)^p$ (distinct from the parametric $u$ of~\eqref{eq:rawmoment}), so that $\E[|Y_s|^p \mid X^{(k-1)}] \le \bar u$ by Minkowski's inequality applied conditionally on $x_s$ (Assumption~\ref{ass:moment}), and let $\hat{m}_k(c_j)$ be the truncated mean of $\{Y_s : x_s \in C_j\}$ with threshold $\tau_j = (\bar u\,n_j/\log(MT))^{1/p}$.
Because the samples have varying means, the i.i.d.\ bound of~\cite{bubeck2013bandits} does not apply verbatim; Lemma~\ref{lem:truncated} in Appendix~\ref{app:truncated} extends it to independent samples with a uniform raw-moment bound, and Corollary~\ref{cor:conditional} performs the conditioning.
For each cell, with conditional probability $\ge 1 - 2/(MT)$:
\begin{equation}\label{eq:cellconc}
  |\hat{m}_k(c_j) - \bar{m}_j| \le 4\,\bar u^{1/p}\!\left(\frac{\log(MT)}{n_j}\right)^{\!(p-1)/p}\!.
\end{equation}
A union bound over all $M$ cells and the tower property give: with probability $\ge 1 - 2/T$, \eqref{eq:cellconc} holds simultaneously for all cells.

\smallskip\noindent\textbf{Step 3: Prediction error (bias + estimation).}
For round $t$ with $x_t \in C_j$, on the good event:
\begin{align}
  |P_t - m(x_t)|
  &= |\hat{m}_k(c_j) - m(x_t)| \notag\\
  &\le \underbrace{|\hat{m}_k(c_j) - \bar{m}_j|}_{\text{estimation}} + \underbrace{|\bar{m}_j - m(c_j)|}_{\le L_H (\sqrt{d} h)^\beta} + \underbrace{|m(c_j) - m(x_t)|}_{\le L_H (\sqrt{d} h)^\beta} \notag\\
  &\le 4\bar u^{1/p}\!\left(\frac{\log(MT)}{n_j}\right)^{\!(p-1)/p} + 2L_H\,(\sqrt{d}\,h)^\beta.\label{eq:npprederr}
\end{align}
Substituting $n_j \ge \mu_0 n (h/2)^d/2$ and writing $C_{\mathrm{est}} = 4\bar u^{1/p}(2^{d+1}/\mu_0)^{(p-1)/p}$:
\begin{equation}\label{eq:nppred2}
  |P_t - m(x_t)| \le C_{\mathrm{est}}\!\left(\frac{\log(MT)}{n\,h^d}\right)^{\!(p-1)/p} + 2L_H\,(\sqrt{d}\,h)^\beta =: \varepsilon_k(h).
\end{equation}

\smallskip\noindent\textbf{Step 4: Per-epoch regret.}
Define the ``good event'' $\mathcal{G}_k$: cell counts~\eqref{eq:chernoff} and estimation~\eqref{eq:cellconc} both hold.
Then $\Pr(\mathcal{G}_k) \ge 1 - 3/T$ (union of the cell-count and estimation events).

\emph{Good event:} By Lemma~\ref{lem:selfbound}, per-round regret $\le L\,\varepsilon_k(h)^2$. Over $n_k$ rounds:
\begin{equation}\label{eq:npgood}
  R_k^{\mathrm{good}} \le n_k\,L\,\varepsilon_k(h)^2
  \le 2n_k\,L\!\left[C_{\mathrm{est}}^2\!\left(\frac{\log(MT)}{n\,h^d}\right)^{\!2(p-1)/p} + 4L_H^2\,d^\beta\,h^{2\beta}\right],
\end{equation}
where we used $(a+b)^2 \le 2(a^2 + b^2)$.

\emph{Bad event:} $R_k^{\mathrm{bad}} \le n_k \cdot 2\sigma_p \cdot 3/T$, by Lemma~\ref{lem:selfbound}.

\smallskip\noindent\textbf{Step 5: Summing over epochs and balancing.}
\emph{Burn-in total:} write $k_1 := \min\{\kappa \ge 2 : 2^{\kappa-2} \ge n_0\}$; epochs $k < k_1$ (including epoch~$1$) are charged at the $2\sigma_p$ cap of Lemma~\ref{lem:selfbound}: they comprise $\sum_{k < k_1} n_k \le 2^{k_1 - 1} \le 4n_0$ rounds, contributing at most $8\sigma_p n_0$.

\emph{Bad-event total:} $\sum_k R_k^{\mathrm{bad}} \le \frac{6\sigma_p}{T}\sum_k 2^{k-1} \le 12\sigma_p$.

\emph{Good-event total (estimation term):} Writing $\alpha = 2(p{-}1)/p$ as before:
\[
  \sum_{k=k_1}^K n_k \cdot \left(\frac{\log(MT)}{n\,h^d}\right)^{\!\alpha}
  \le (\log(MT))^\alpha\,(h^d)^{-\alpha}\sum_{k=2}^K \frac{2^{k-1}}{(2^{k-2})^\alpha}.
\]
This sum has the same form as~\eqref{eq:epochgeo}--\eqref{eq:geosum}, so the same bound gives $O(T^{(2-p)/p}\cdot h^{-2d(p-1)/p}\cdot(\log T)^\alpha \cdot K)$, i.e.\ $\Otil(T^{(2-p)/p}h^{-2d(p-1)/p})$.

\emph{Good-event total (bias term):} $\sum_k n_k \cdot h^{2\beta} = T\,h^{2\beta}$.

Combining, and writing $K = \lceil\log_2(T{+}1)\rceil$ for the epoch count carried by~\eqref{eq:geosum}, we obtain
\[
  R_T \;\lesssim\; L\bigl[C_{\mathrm{est}}^2\,(\log(MT))^\alpha\,K\,T^{(2-p)/p}\,h^{-2d(p-1)/p} + L_H^2\,d^\beta\,T\,h^{2\beta}\bigr] + \sigma_p + \sigma_p n_0,
\]
where $\lesssim$ hides absolute constants only --- the factors $2$ and $4$ of~\eqref{eq:npgood}.
\begin{sloppypar}
The burn-in term is $O(\sigma_p\,h^{-d}\log(h^{-d}T)/\mu_0)$. At the optimized width below, $h^{-d} = T^{d(p-1)/(\beta p + d(p-1))}$, so~\eqref{eq:npfinal} dominates it with exponent margin $\beta(2-p)/(\beta p + d(p-1)) > 0$ (at $p{=}2$, the same polylogarithmic order); it too is absorbed into $\Otil$.
\end{sloppypar}

\emph{Balancing.} The prefactor $(\log(MT))^\alpha K$ on the estimation term is polylogarithmic in $T$ at any width with $h^{-1} \le T$ --- which~\eqref{eq:hopt} satisfies --- since $M \le (2/h)^d$ gives $\log(MT) \le (d{+}1)\log(2T)$, at the price of $d$-dependent constants; it changes neither the balance point nor the resulting exponent, so we balance the powers of $T$ and $h$ and absorb it into the $\Otil$ of~\eqref{eq:npfinal}. Setting the two terms equal, from $T^{(2-p)/p}\,h^{-2d(p-1)/p} = T\,h^{2\beta}$
we get $h^{2\beta + 2d(p-1)/p} = T^{-2(p-1)/p}$.
Solving:
\begin{equation}\label{eq:hopt}
  h = T^{-(p-1)/(\beta p + d(p-1))}.
\end{equation}
Substituting into $T\,h^{2\beta}$:
\[
  T\,h^{2\beta} = T^{1 - 2\beta(p-1)/(\beta p + d(p-1))}.
\]
Therefore:
\begin{equation}\label{eq:npfinal}
  R_T = \Otil\!\left(T^{\;1 - 2\beta(p-1)/(\beta p + d(p-1))}\right),
\end{equation}
with problem constants $L, L_H, \sigma_p, B_m, \mu_0, d, \beta$ and the polylogarithmic factors $(\log(MT))^\alpha K$ absorbed into $\Otil$.

\emph{Sanity checks.} $p{=}2$: $h = T^{-1/(2\beta+d)}$, $R_T = \Otil(T^{d/(2\beta+d)})$, recovering Stone~\cite{stone1982optimal}. $p \to 1^+$: $h \to T^0 = 1$ (single cell, no spatial resolution), $R_T \to \Otil(T)$. \qed

\section{Lower Bound}\label{sec:lower}

\begin{proposition}\label{prop:lower}
Fix parameters as in Theorem~\ref{thm:nonparametric} with $\beta \in (0,1]$ (the effective range of Assumption~\ref{ass:smooth}; for $\beta > 1$ the class contains only constant $m$ and the statement below would fail), and assume the mild nondegeneracy $(2L\sigma_p)^p \ge 8$. Then for every nonanticipating algorithm there are instances satisfying Assumptions~\ref{ass:density}--\ref{ass:regular} with these parameters on which
\[R_T = \Omega\!\left(T^{\;1 - 2\beta(p-1)/(\beta p + d(p-1))}\right).\]
\end{proposition}

\noindent The instances constructed in the proof satisfy Assumption~\ref{ass:regular} as well, so the lower bound holds even within that regular subclass; the proposition itself does not assume it.
They also have \emph{finite} (though $T$-dependent, and unbounded along the sequence) variance: the hardness is driven by the absence of a uniform second-moment bound over the class, not by literal infinite variance of any single instance.
Their conditional noise laws also vary with the context --- the mean perturbation is carried by mixture weights on a fixed support --- and whether the same rates remain minimax within the context-independent subclass is open (Q5, Section~\ref{sec:discussion}).

\begin{proof}
\textbf{Step 1 (Assouad reduction).}
In the constructed instances we take $x_t$ uniform on $[0,1]^d$ --- admissible, since any density lower bound necessarily has $\mu_0 \le 1$; uniformity pins the per-subcube round counts from \emph{both} sides, which the argument needs (an upper count for the KL budget, a lower count for the regret).
Partition $[0,1]^d$ into $M = h^{-d}$ subcubes of side~$h$ (choose $h$ with $1/h$ an integer, e.g.\ dyadic --- the rounding affects constants only).
For each $\theta \in \{-1,+1\}^M$, define $m_\theta(x) = 2\varepsilon + \sum_{j=1}^M \theta_j\,\varepsilon\,\psi_j(x)$,
where $\psi_j$ is a bump with $\psi_j \equiv 1$ on the \emph{inner half-cube} of subcube~$j$ (side $h/2$, the ``core''), $\psi_j = 0$ outside subcube~$j$, and $0 \le \psi_j \le 1$, tapering over a shell of width $h/4$.
Each taper alone gives $\varepsilon\psi_j$ H\"older seminorm at most $\varepsilon\,(4/h)^\beta$, but the binding pairs are \emph{cross-support}: $m_\theta$ equals the constant $2\varepsilon$ on every subcube boundary, so for $x, y$ in different subcubes, writing $r_x, r_y$ for the distances from $x, y$ to their own subcubes' boundaries along the segment $[x, y]$ (so $r_x + r_y \le \|x - y\|$), we get $|m_\theta(x) - m_\theta(y)| \le \varepsilon\,(4/h)^\beta(r_x^\beta + r_y^\beta) \le 2^{1-\beta}\varepsilon\,(4/h)^\beta\|x - y\|^\beta$ by concavity of $r \mapsto r^\beta$ --- the worst case being adjacent cores of opposite sign, an increment of $2\varepsilon$ over distance $h/2$. The constraint that each $m_\theta$ be $(\beta, L_H)$-H\"older therefore requires $\varepsilon \le 2^{-(\beta+1)} L_H h^\beta =: c_\beta L_H h^\beta$ (strictly tighter than the single-bump requirement $4^{-\beta}L_H h^\beta$ for every $\beta < 1$).
(The constant offset $2\varepsilon$, common to all hypotheses, is what the mixture realization of Step~3 produces; it affects neither the H\"older constraint nor the regret gaps.)
On the core, the perturbation is exactly $\pm\varepsilon$; outside it we simply discard the (nonnegative) regret contribution.
The \emph{reverse} self-bounding property, applied conditionally on the context, supplies a clipped per-round lower bound: with
$c_\varepsilon := \tfrac12\operatorname*{ess\,inf}_x \inf_{|s|\le\varepsilon}\,\bigl[f_\xi(s \mid x)+f_\zeta(s \mid x)\bigr]$,
any price $a$ incurs conditional expected regret at least $c_\varepsilon \min\bigl((a - m_\theta(x))^2,\, \varepsilon^2\bigr)$ at context $x$ (integrate $h_x'$ over $[0, \min(|a - m_\theta(x)|, \varepsilon)]$).
For the instances constructed in Step~3 below, the conditional noise densities are at least $L/2$ on $[-\varepsilon, \varepsilon]$, uniformly in $x$ and in the hypothesis, so $c_\varepsilon \ge L/4 > 0$; no regularity hypothesis on the environment is needed.
(The constant $c_\varepsilon$ is local to this proof.)

\emph{Formal reduction.}
Fix an algorithm; let $P_\theta$ denote the law of the full trajectory under $m_\theta$ and $R(\theta)$ its expected regret.
Since the cores are disjoint, $R(\theta) \ge c_\varepsilon \sum_{j=1}^M \E_\theta[F_\theta^{(j)}]$, where $F_\theta^{(j)}$ is the realized sum of $\min\bigl((P_t - m_\theta(x_t))^2, \varepsilon^2\bigr)$ over rounds with $x_t \in \mathrm{core}_j$.
Pair $\theta$ with $\theta^{(j)}$ (bit $j$ flipped): the two regression functions differ by exactly $2\varepsilon$ on $\mathrm{core}_j$, so for every price $a$ and $x \in \mathrm{core}_j$,
$\min((a - m_\theta(x))^2, \varepsilon^2) + \min((a - m_{\theta^{(j)}}(x))^2, \varepsilon^2) \ge \varepsilon^2$;
hence, trajectory-wise, $F_\theta^{(j)} + F_{\theta^{(j)}}^{(j)} \ge \varepsilon^2 N_j$, where $N_j$ (the number of core-$j$ rounds) is a function of the design alone.
Integrating against the minimum measure $\mu_j := \min(dP_\theta, dP_{\theta^{(j)}})$~\cite[Section~2.4]{tsybakov2009introduction}:
\begin{equation}\label{eq:assouad}
  \E_\theta[F_\theta^{(j)}] + \E_{\theta^{(j)}}[F_{\theta^{(j)}}^{(j)}]
  \;\ge\; \int \bigl(F_\theta^{(j)} + F_{\theta^{(j)}}^{(j)}\bigr)\, d\mu_j
  \;\ge\; \varepsilon^2 \int N_j \, d\mu_j .
\end{equation}
Under the uniform design $N_j \sim \mathrm{Bin}(T, 2^{-d}h^d)$; writing $\bar N := 2^{-d}Th^d$ and using $\mu_j(A) \ge P_\theta(A) - \mathrm{TV}(P_\theta, P_{\theta^{(j)}})$ for every event $A$,
\[
  \int N_j\, d\mu_j \;\ge\; \frac{\bar N}{2}\Bigl( P_\theta\bigl(N_j \ge \tfrac{\bar N}{2}\bigr) - \mathrm{TV}(P_\theta, P_{\theta^{(j)}}) \Bigr)
  \;\ge\; \frac{\bar N}{8}
\]
for all large $T$, provided $\mathrm{TV}(P_\theta, P_{\theta^{(j)}}) \le 1/2$ --- which Step~3's budget enforces --- since $P_\theta(N_j \ge \bar N/2) \ge 1 - e^{-\bar N/8} \to 1$.
Averaging~\eqref{eq:assouad} over $\theta \in \{-1,+1\}^M$ (the map $\theta \mapsto \theta^{(j)}$ is a bijection) and summing over $j$:
\begin{equation}\label{eq:assouad2}
  \sup_\theta R(\theta) \;\ge\; \frac{1}{2^M}\sum_\theta R(\theta)
  \;\ge\; c_\varepsilon \sum_{j=1}^M \frac{\varepsilon^2 \bar N}{16}
  \;=\; \frac{c_\varepsilon}{2^{d+4}}\, T\,\varepsilon^2 .
\end{equation}

\smallskip\noindent\textbf{Step 2 (Full-feedback KL).}
Fix subcube~$j$ and a context $x$ therein.
The two hypotheses $\theta_j = \pm 1$ induce two conditional laws $P_x, P'_x$ for $V_t$, and (independently, given $x_t$) the same pair for $W_t$.
By conditional independence of $V_t$ and $W_t$ given $x_t$, the per-round divergence of the observed pair is
\[
  \mathrm{KL}\bigl(P_x^{(V,W)} \,\|\, P'^{(V,W)}_x\bigr)
  = 2\,\mathrm{KL}(P_x \| P'_x).
\]
This is exactly twice the KL from observing a single trader---full feedback provides no order-of-magnitude advantage.

\smallskip\noindent\textbf{Step 3 (Fixed-support mixture construction).}
Set $a=(\sigma_p^p/(4\varepsilon))^{1/(p-1)}$ and let $U_0,U_a$ be uniform densities of
width $1/L$ centered at $0$ and $a$; since $a\to\infty$ as $T\to\infty$ they are disjoint
for all large $T$. For $x$ in subcube $j$, define the conditional law of $V_t$ given
$x_t=x$ as $(1-w_\theta(x))\,U_0+w_\theta(x)\,U_a$ with
$w_\theta(x)=\bigl[2\varepsilon+\theta_j\varepsilon\psi_j(x)\bigr]/a\in[\varepsilon/a,3\varepsilon/a]$,
and independently the same for $W_t$. The regression function is the conditional mean,
$m_\theta(x)=w_\theta(x)\,a=2\varepsilon+\theta_j\varepsilon\psi_j(x)$, which is
$(\beta,L_H)$-H\"older for $\varepsilon\le c_\beta L_H h^\beta$ (the cross-support constraint derived above); the noise
$\xi_t=V_t-m_\theta(x_t)$ is conditionally centered \emph{by construction}, with
conditional density bounded by $L$. Its conditional $p$-th moment is at most
$w_\theta(x)\,(a+\tfrac{1}{2L})^p+(\tfrac{1}{2L}+3\varepsilon)^p
\le \tfrac34\,\sigma_p^p\,(1+o(1))+(2L)^{-p}(1+o(1))\le\sigma_p^p$ for all large $T$,
provided $(2L\sigma_p)^p \ge 8$, as assumed.
(This nondegeneracy is provably mild: \emph{any} zero-mean density bounded by $L$ satisfies
$\sigma_p^p \ge (2L)^{-p}/(p+1)$---the uniform distribution on $[-\tfrac{1}{2L},\tfrac{1}{2L}]$
is the minimizer, by the bathtub principle---so every admissible class already has
$(2L\sigma_p)^p \ge 1/(p+1)$, and the proviso excludes only a bounded strip above this forced
floor.)
Because the two hypotheses are mixtures of the \emph{same} two components, the component
indicator is sufficient and the conditional KL is \emph{exactly} the Bernoulli KL of the
weights: with $w = w_\theta(x)$, $w' = w_{\theta'}(x)$,
\begin{equation}\label{eq:klmoment}
  \mathrm{KL}(P_x \| P'_x)
  \le \frac{(w-w')^2}{w'(1-w')}
  \le \frac{8\,\varepsilon\,\psi_j(x)^2}{a}
  = O\bigl(\varepsilon^{p/(p-1)}\,\sigma_p^{-p/(p-1)}\bigr)
  \quad\text{uniformly in } x;
\end{equation}
the H\"older taper only shrinks the weight gap and never moves the support.
Finally, the recentred lower bump is a width-$1/L$ uniform centered at
$-m_\theta(x)\in[-3\varepsilon,-\varepsilon]$; since $3\varepsilon<1/(4L)$ for all large $T$,
$[-\varepsilon,\varepsilon]$ lies inside its support, so
$f_\xi(s\mid x)\ge L(1-w_\theta(x))\ge L/2$ there. The constructed instances therefore
satisfy Assumptions~\ref{ass:density}--\ref{ass:smooth} \emph{and}
Assumption~\ref{ass:regular} with $c_\varepsilon\ge L/4$, and~\eqref{eq:klmoment} holds; Steps~1,~2
and~4 proceed as before.
By the chain rule for KL divergence, the trajectory laws $P_\theta, P_{\theta^{(j)}}$ differ only through rounds landing in subcube~$j$, and with per-round divergence $2\,\mathrm{KL}(P_{x_t} \| P'_{x_t})$ from the pair $(V_t, W_t)$ (Step~2),
$\mathrm{KL}(P_\theta \| P_{\theta^{(j)}}) \le 2\,\E_\theta\bigl[\#\{t : x_t \in \text{subcube } j\}\bigr] \cdot \sup_x \mathrm{KL}(P_x \| P'_x) = 2Th^d \sup_x \mathrm{KL}(P_x \| P'_x)$ under the uniform design --- the chain rule handles the random visit count exactly, in expectation.
By Pinsker, $\mathrm{TV}(P_\theta, P_{\theta^{(j)}}) \le \sqrt{\mathrm{KL}/2} \le 1/2$ whenever $Th^d \sup_x \mathrm{KL} \le 1/4$; by~\eqref{eq:klmoment} (uniform in $x$) this holds whenever $\varepsilon \le c'(Th^d)^{-(p-1)/p}$, enforcing the budget required by Step~1's formal reduction.
We therefore \emph{choose} the perturbation as large as both constraints allow,
\begin{equation}\label{eq:epsbarrier}
  \varepsilon \asymp \min\bigl\{\, c_\beta L_H h^\beta,\; c'\,(Th^d)^{-(p-1)/p} \,\bigr\}.
\end{equation}

\smallskip\noindent\textbf{Step 4 (Optimization).}
Balance the two caps in~\eqref{eq:epsbarrier}: $(Th^d)^{-(p-1)/p} = h^\beta$, i.e.,
$h^{\beta + d(p-1)/p} = T^{-(p-1)/p}$, giving $h = T^{-(p-1)/(\beta p + d(p-1))}$.
Then $\varepsilon = h^\beta = T^{-\beta(p-1)/(\beta p + d(p-1))}$.
From~\eqref{eq:assouad2}:
\[
  R_T \ge \frac{c_\varepsilon}{2^{d+4}}\,T\,\varepsilon^2
  = \Omega\!\left(T^{\;1 - 2\beta(p-1)/(\beta p + d(p-1))}\right).
\]
\emph{Sanity checks.} $p{=}2$: exponent $= d/(2\beta{+}d)$, matching Stone~\cite{stone1982optimal}. $p \to 1^+$: exponent $\to 1$, matching trivial barrier. \qed
\end{proof}

\begin{proposition}[Parametric lower bound in $T$]\label{prop:paramlower}
Fix $p \in (1,2)$, $L \ge 1$, $\sigma_p$ with $(2L\sigma_p)^p \ge 8$, and $B > 0$.
For $d = 1$, every nonanticipating algorithm admits instances of the parametric model of Theorem~\ref{thm:parametric} (with $\lambda = 1$) on which, for all $T$ large enough (depending on $p, L, \sigma_p, B$),
\[
  R_T \;\ge\; c(p)\,L\,\sigma_p^2\,T^{(2-p)/p}, \qquad c(p) > 0.
\]
For fixed $d \ge 2$ the same $T$-rate holds with constants depending on $d$.
\end{proposition}

\begin{proof}
Let $x_t \equiv 1$, so $\|x_t\| \le 1$ and $\Sigma = 1 \succeq \lambda = 1$, and consider the two linear hypotheses $\phi_+ := 3\varepsilon$ and $\phi_- := \varepsilon$, so $m_\pm \equiv \phi_\pm$: a gap of $2\varepsilon$, with $|\phi_\pm| \le B$ for all large $T$.
Under hypothesis $\pm$, give $V_t$ (and independently $W_t$) the conditional law of Step~3 of the proof of Proposition~\ref{prop:lower} with a single site: $(1 - w_\pm)\,U_0 + w_\pm\,U_a$, where $a = (\sigma_p^p/(4\varepsilon))^{1/(p-1)}$ and $w_\pm := m_\pm/a \in [\varepsilon/a,\,3\varepsilon/a]$.
Exactly as there, for all large $T$: the noise is conditionally centered with density bounded by $L$ and at least $L/2$ on $[-\varepsilon, \varepsilon]$ (so the clipped bound of that proof's Step~1 applies with $c_\varepsilon \ge L/4$), its $p$-th moment is at most $\sigma_p^p$ under the proviso $(2L\sigma_p)^p \ge 8$, and the per-trader per-round KL is the Bernoulli KL of the weights:
\[
  \mathrm{KL} \;\le\; \frac{(w_+ - w_-)^2}{w_-(1 - w_-)} \;\le\; \frac{8\varepsilon}{a}
  \;=\; 8 \cdot 4^{1/(p-1)}\Bigl(\frac{\varepsilon}{\sigma_p}\Bigr)^{\!p/(p-1)} \;=:\; \kappa_p\Bigl(\frac{\varepsilon}{\sigma_p}\Bigr)^{\!p/(p-1)}.
\]
Full feedback doubles this (Step~2 there), and the chain rule over $T$ rounds gives $\mathrm{KL}(P_+ \| P_-) \le 2\kappa_p\,T\,(\varepsilon/\sigma_p)^{p/(p-1)}$.
Choose $\varepsilon := \sigma_p\,(16\kappa_p T)^{-(p-1)/p}$, so that $\mathrm{KL}(P_+ \| P_-) \le 1/8$ and, by Pinsker, $\mathrm{TV}(P_+, P_-) \le 1/4$.
Since $m_+ - m_- = 2\varepsilon$ everywhere, pointwise $\min((P_t - m_+)^2, \varepsilon^2) + \min((P_t - m_-)^2, \varepsilon^2) \ge \varepsilon^2$ for every price; summing over the $T$ rounds and integrating against the minimum measure $\min(dP_+, dP_-)$, of total mass at least $1 - \mathrm{TV} \ge 3/4$:
\[
  R_T(+) + R_T(-) \;\ge\; c_\varepsilon\,\varepsilon^2\,T\,(1 - \mathrm{TV}) \;\ge\; \tfrac{3}{16}\,L\,\varepsilon^2\,T,
\]
so $\max_\pm R_T(\pm) \ge \tfrac{3}{32}\,L\,\varepsilon^2\,T = c(p)\,L\,\sigma_p^2\,T^{(2-p)/p}$ with $c(p) = \tfrac{3}{32}(16\kappa_p)^{-2(p-1)/p}$.
For $d \ge 2$: draw $x_t$ uniformly from $\{e_1, \dots, e_d\}$ (so $\|x_t\| = 1$ and $\Sigma = I_d/d \succeq \lambda I_d$ with $\lambda = 1/d$), set $\phi_+ = (3\varepsilon, 0, \dots, 0)$, $\phi_- = (\varepsilon, 0, \dots, 0)$, use the pair above at $x = e_1$, and at $x = e_i$, $i \ge 2$, the centered component $U_0$ (identical under both hypotheses, hence zero KL contribution; its $p$-th moment is $\le (2L)^{-p} \le \sigma_p^p$ by the proviso).
Only the $\mathrm{Bin}(T, 1/d)$ rounds at $e_1$ carry information or a regret gap; repeating the argument with $\varepsilon = \sigma_p(16\kappa_p T/d)^{-(p-1)/p}$ and a Chernoff bound on the visit count gives $R_T = \Omega_d\bigl(L\,\sigma_p^2\,(T/d)^{(2-p)/p}\bigr)$.
\end{proof}

\begin{remark}[No dimension claim]\label{rem:paramlowerd}
We do not claim a matching dimension dependence: in the conditional model the per-round divergence of the mixture construction is of order $\varepsilon^{p/(p-1)}\sigma_p^{-p/(p-1)}$ in the local mean gap $\varepsilon$ (super-quadratic for $p<2$; quadratic only in the limit $p \to 2^-$), and the optimal $d$-dependence of the parametric lower bound is left open.
(The exponent of Proposition~\ref{prop:lower} also tends to $(2-p)/p$ as $\beta \to \infty$, consistent with Proposition~\ref{prop:paramlower}; we do not use that limit as a proof device, since bump functions with $\beta > 1$ increment-H\"older regularity do not exist.)
\end{remark}

\section{Experiments}\label{sec:experiments}

We validate the rates of Theorems~\ref{thm:parametric} and~\ref{thm:nonparametric} in simulation.
The acceptance criteria were fixed in advance of the runs; Appendix~\ref{app:experiments} gives the full protocol --- every estimator, hyperparameter, seed, and acceptance rule (simulations are \texttt{numpy}-only, and the full pre-specified run takes under a minute on a laptop).

\smallskip\noindent\textbf{Setup.}
Noise: Student-$t(\nu)$ with $\nu \in \{1.2, 1.5, 1.8\}$, plus a symmetrized Lomax$(1.5)$ (Pareto type~II, which unlike a pure Pareto has positive density at zero, per Assumption~\ref{ass:regular}).
Geometries: parametric with $d = 5$, contexts uniform on the sphere, $\|\phi\| = 1$; nonparametric with $d \in \{1,2\}$ and a fixed Lipschitz $m$ with $B_m = 0.3$.
All instances use context-independent noise---the special case of Assumptions~\ref{ass:density}--\ref{ass:moment} in which the conditional law does not depend on $x$.
The epoch-truncated algorithm runs at $p_{\mathrm{alg}} = \nu - 0.05$ with oracle $(p, \sigma_p, B)$; the baseline uses the identical epoch schedule with untruncated least squares (parametric) or plain cell means (nonparametric); $T = 10^6$, $20$ seeds.
Per-round \emph{expected} regret is computed exactly through the closed form of Lemma~\ref{lem:selfbound} (Appendix~\ref{app:experiments}), which removes the realized-gain variance that heavy tails would otherwise inject into the curves.

\begin{figure}[t]
  \centering
  \includegraphics[width=\textwidth]{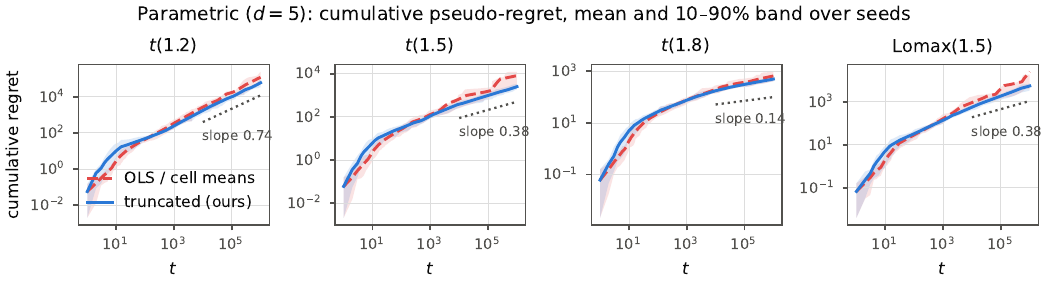}
  \caption{Parametric setting ($d = 5$): cumulative expected regret (mean and 10--90\% band over $20$ seeds, log-log). Dotted guides show the exponent of Theorem~\ref{thm:parametric} at $p_{\mathrm{alg}}$.}
  \label{fig:param}
\end{figure}

\begin{figure}[t]
  \centering
  \includegraphics[width=\textwidth]{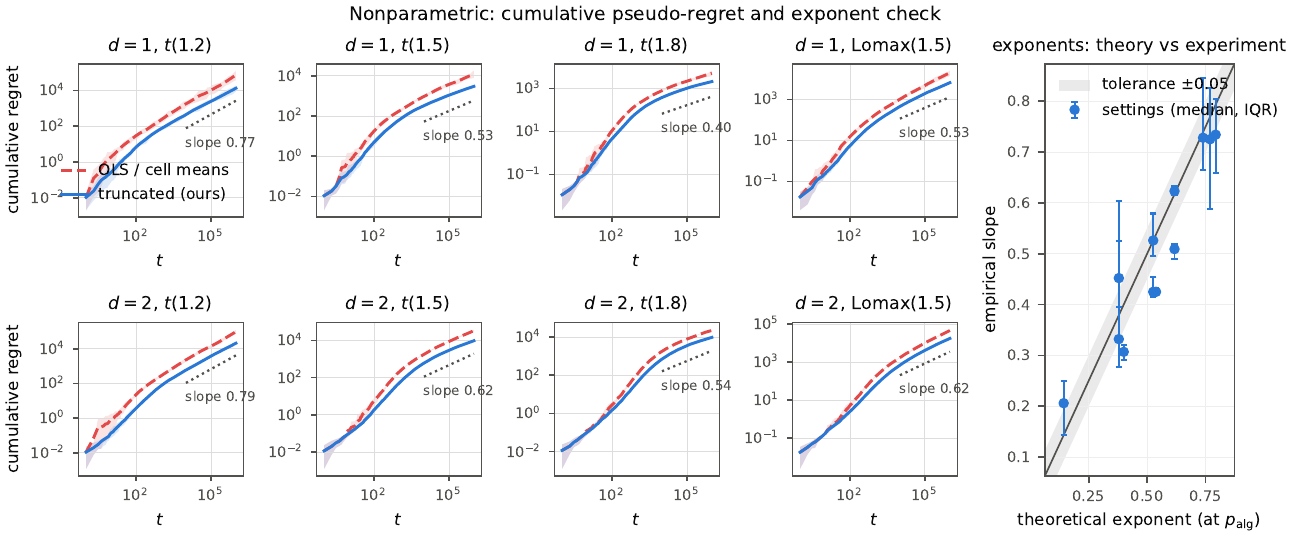}
  \caption{Cumulative expected regret per noise model in the nonparametric settings ($d \in \{1,2\}$); the right panel shows fitted last-decade slopes for all twelve settings (parametric and nonparametric) against the exponents of Theorems~\ref{thm:parametric} and~\ref{thm:nonparametric} at $p_{\mathrm{alg}}$, with the pre-specified $+0.05$ acceptance line (band drawn at $\pm 0.05$ for reference).}
  \label{fig:nonparam}
\end{figure}

\smallskip\noindent\textbf{Rate check (Figures~\ref{fig:param}--\ref{fig:nonparam}).}
In $10$ of $12$ settings the fitted last-decade slope of the truncated algorithm is at most the exponent of Theorems~\ref{thm:parametric} and~\ref{thm:nonparametric} at $p_{\mathrm{alg}}$ plus the pre-specified $0.05$ tolerance (slopes below the line are compatible with the worst-case theory and pass), and in every setting the truncated algorithm's mean final regret is below the untruncated baseline's.
The baseline's failure mode is upper-tail instability rather than a shifted median: under $t(1.5)$ its worst-seed final regret is $16\times$ that of the truncated algorithm.

\smallskip\noindent\textbf{The two exceptions, at longer horizons (Figure~\ref{fig:horizon}).}
Both exceptions are parametric---$t(1.8)$ and Lomax$(1.5)$---the two cells where the asymptotic exponent $\theta$ is smallest relative to the bound's lower-order terms.
Extending these two cells (and $t(1.5)$ as a control) to $T = 10^7$, the mean-curve local log-log slope declines over the measured windows---$0.206 \to 0.186$ for $t(1.8)$ (against $\theta = 0.143$) and $0.499 \to 0.480$ for Lomax (against $\theta = 0.379$)---while the control sits on its exponent ($0.384 \to 0.411$ around $\theta = 0.379$); per-seed slope changes are heterogeneous ($12/20$ and $10/20$ seeds decline, respectively), so we attach no statistical significance to the decline itself.
A decaying finite-horizon excess is consistent with the bound's lower-order terms; the data alone cannot exclude asymptotic exponents between $\theta$ and the last measured slope, and we make no such claim.
We also deliberately do not attribute the decline to a specific mechanism --- the implementation tunes $\tau$ with $\log(dT)$ fixed at the full horizon, so within a run the bound's logarithmic factor is a constant multiplier, and additive epoch transients contribute as well.

\begin{figure}[t]
  \centering
  \includegraphics[width=\textwidth]{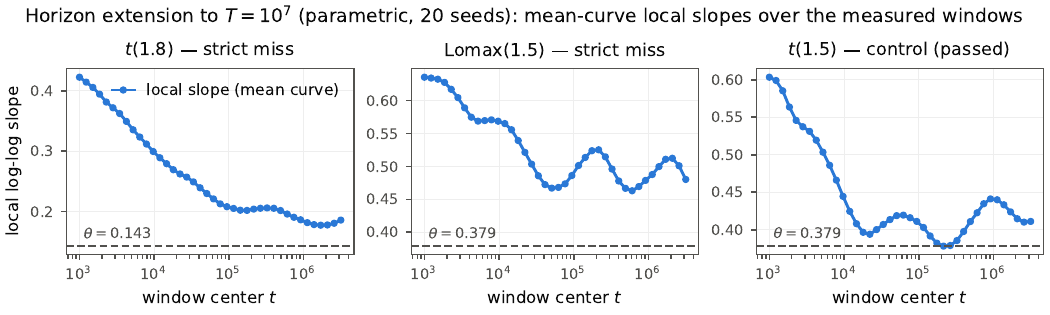}
  \caption{Horizon extension to $T = 10^7$ (parametric, $20$ seeds): local log-log slope of the mean cumulative regret over sliding one-decade windows. In the two settings that miss the strict finite-$T$ acceptance line, the mean-curve slope declines over the measured windows (dashed: $\theta$); the control sits on its $\theta$.}
  \label{fig:horizon}
\end{figure}

\section{Parameter-Free Adaptivity}\label{sec:adaptive}

The algorithms of Theorems~\ref{thm:parametric} and~\ref{thm:nonparametric} consume the class parameters: the truncation levels need $(p, \sigma_p)$ and $B$ (resp.\ $B_m$), and the cell width needs $(p, \beta)$.
Under bandit feedback, adaptivity to unknown moment parameters is provably costly~\cite{ashutosh2021bandit,genalti2024adaptive}.
We now show that under full feedback the parametric problem admits a \emph{fully parameter-free} algorithm at the oracle rate: \textbf{the price of tail-adaptivity is a price of exploration, not of estimation}.
The bandit impossibilities are information-theoretic consequences of partial feedback: exploration is unavoidable there, and tail-blind exploration is provably wasteful for \emph{every} algorithm.
Full feedback removes the need to explore; once the estimator itself is parameter-free, a greedy scheme has nothing parameter-dependent left to compute.

\subsection{The parametric case: adaptivity is free}\label{sec:adaptive:param}

\smallskip\noindent\textbf{Algorithm (MoM-greedy).}
Epochs as in Section~\ref{sec:proof:param}.
Set $k_T := \lceil 8\log(dT) \rceil$.
For epoch $k \ge 2$, using the $n = 2^{k-2}$ samples of the previous epoch: for each coordinate $j$, split $\{[S_s]_j\}_{s=1}^n$ (scores $S_s = x_s Y_s$) into $k_T$ contiguous blocks and let $\hat\mu_j$ be the \emph{median of the block means}; form $\hat\Sigma = n^{-1}\sum_s x_s x_s^\top$ and $\hat\phi_k = \hat\Sigma^{+}\hat\mu$ (pseudoinverse if singular); play $P_t = x_t^\top \hat\phi_k$.
Epochs with $n < 2k_T$ retain the previous estimate ($\hat\phi = 0$ initially); their rounds are absorbed by the regret cap.
The algorithm's inputs are $d$ and $T$ alone: no truncation level, no confidence radius, no projection, and no bound on $\|\phi\|$.

The guarantee is Theorem~\ref{thm:adaptive}, stated in Section~\ref{sec:results}; its proof rests on three observations.

\begin{proof}[Proof sketch of Theorem~\ref{thm:adaptive} (full proof in Appendix~\ref{app:adaptive})]
Three observations. (i)~The median-of-means estimator is parameter-free: with $k_T$ blocks, Lemma~\ref{lem:mom} gives per-coordinate deviation $O\bigl(m_p^{1/p}(\log(dT)/n)^{(p-1)/p}\bigr)$ with $m_p \le 2^p(B+\sigma_p)^p$ --- the same shape as~\eqref{eq:percoord}; $p$ and $\sigma_p$ appear only in the analysis.
(ii)~The algorithm is greedy, so no confidence radius is ever computed --- the parameter-dependence of exploration bonuses, which drives the bandit impossibilities, has no analogue here.
(iii)~Since $g \ge 0$, the per-round regret of \emph{any} price is at most $\E[g(m,V,W)] \le 2\sigma_p$; wild prices caused by an ill-conditioned $\hat\Sigma$ or by epochs shorter than $2k_T$ are absorbed by this cap in the same slot as the existing bad-event accounting --- this is what allows the known-$B$ assumption to be dropped rather than avoided.
The proof of Theorem~\ref{thm:parametric} then goes through with~\eqref{eq:percoord} replaced by Lemma~\ref{lem:mom}.
\end{proof}

\smallskip\noindent\textbf{Empirically, the parameter-free algorithm is also the better algorithm (Figure~\ref{fig:adaptive}).}
Running MoM-greedy against the oracle-tuned truncation of Theorem~\ref{thm:parametric} and the OLS baseline \emph{on identical sample paths} (common random numbers; protocol in Appendix~\ref{app:experiments}), its mean cumulative regret is below the oracle-tuned algorithm's at every last-decade checkpoint in all four noise settings, with final-regret ratios from $0.61$ down to $0.05$.
This is not a strawman effect but a boundary phenomenon of the tuned algorithm itself: at $p_{\mathrm{alg}} = \nu - 0.05$ the closed-form moment carries $\Gamma((\nu - p_{\mathrm{alg}})/2) = \Gamma(0.025) \approx 40$, so the theory-prescribed truncation level $\tau = (un/\log(dT))^{1/p}$ is enormous and the truncated estimator degenerates toward OLS --- visible in the panels where the two curves track each other.
Median-of-means has no constant to mis-set: the tuned algorithm inherits the conservatism of its worst-case constants near the moment boundary, and the parameter-free algorithm does not.

\begin{figure}[t]
  \centering
  \includegraphics[width=\textwidth]{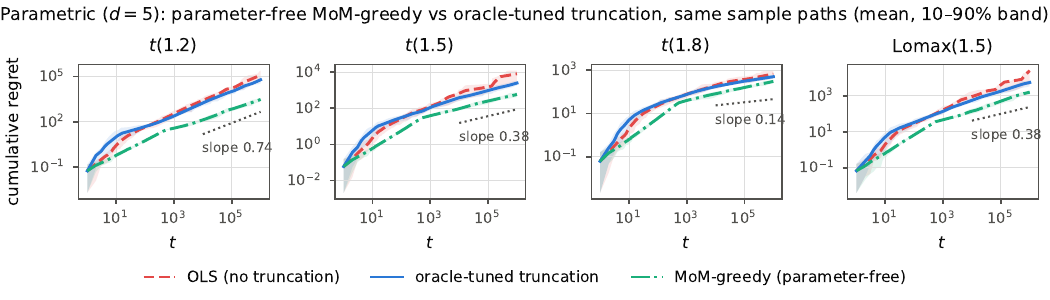}
  \caption{Parametric setting ($d=5$): the parameter-free MoM-greedy algorithm of Theorem~\ref{thm:adaptive} vs.\ the oracle-tuned algorithm of Theorem~\ref{thm:parametric} and OLS, on identical sample paths ($20$ seeds; mean and 10--90\% band). MoM-greedy dominates the oracle-tuned truncation pointwise over the last decade in all four settings.}
  \label{fig:adaptive}
\end{figure}
\begin{figure}[!t]
  \centering
  \includegraphics[width=0.8\textwidth]{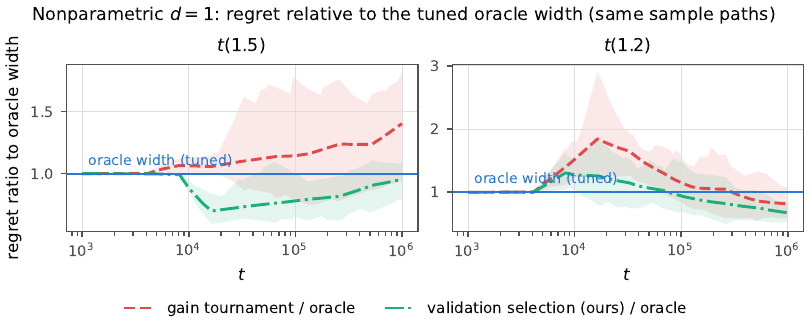}
  \caption{Nonparametric $d = 1$: cumulative regret relative to the tuned oracle width, per-seed paired ratios under common random numbers ($20$ seeds; mean and 10--90\% band). The validation selector of Theorem~\ref{thm:adaptivenp} finishes below the oracle line in both cells (after a transient excursion above it around $t \approx 10^4$ at $t(1.2)$); the gain-based tournament of Remark~\ref{rem:whyvalidation} finishes above it at $t(1.5)$ (ratio $1.40$) but below it at $t(1.2)$ (ratio $0.82$; ratios are endpoints of the plotted mean-of-per-seed-ratio curves), the validation selector improving on the tournament in both cells.}
  \label{fig:validation}
\end{figure}

\subsection{The nonparametric case: validation-based width selection}\label{sec:adaptive:np}

In the nonparametric model the genuinely parameter-dependent object is the cell width $h = h(p, \beta, d)$.
Adaptivity is obtained by running all widths on a dyadic grid and selecting by median-of-means comparisons of \emph{validation losses} against the observable pseudo-label $Y_t = (V_t + W_t)/2$: for two candidates $\hat m_h, \hat m_{h'}$ the per-round difference is
\[
  \bigl(\hat m_h(x_t) - Y_t\bigr)^2 - \bigl(\hat m_{h'}(x_t) - Y_t\bigr)^2
  = \bigl[(\hat m_h - m)^2 - (\hat m_{h'} - m)^2\bigr](x_t) \;-\; 2\,\eta_t\,\Delta(x_t),
  \qquad \Delta := \hat m_h - \hat m_{h'},
\]
so the noise's square cancels exactly, the noise enters \emph{linearly, damped by the candidate gap}, and the expectation is the exact excess risk.

This yields Theorem~\ref{thm:adaptivenp} (stated in Section~\ref{sec:results}); Appendix~\ref{app:adaptive} gives the algorithm and full proof.

\begin{remark}[Why validation losses, and not realized gains]\label{rem:whyvalidation}
Selecting by comparisons of realized gains is weaker in our analysis.
The gain difference of two prices $\Delta$ apart is $\pm Z\,\ind\{S\}$ with $\Pr(S \mid x) \le 4L\Delta(x)$ (Lemma~\ref{lem:paired}): a rare event of probability $\sim \Delta$ carrying a $\Theta(1)$-scale payoff, so its gap-localized $p$-th central moment shrinks only \emph{linearly} in the price gap (Lemma~\ref{lem:paired}(iii); at $p = 2$, a Bernstein condition with exponent $\tfrac12$), and equating the resulting comparison radius with the candidate gap yields a smallest resolvable price gap of order $T^{-(p-1)/(2p-1)}$, hence a cumulative selection cost of order $T^{1/(2p-1)}$: the tournament built on gain comparisons that we analyzed in full in an earlier version of this manuscript incurred an $\Otil(T^{1/(2p-1)})$ selection cost. That analysis is not reproduced here, and the rate should be read as a heuristic calibration of the gain interface (its moment ingredients, Lemma~\ref{lem:paired}(i)--(iv), \emph{are} proved in Appendix~\ref{app:adaptive}), not as a theorem of this paper.
Within the increment-H\"older scale of Assumption~\ref{ass:smooth} (effectively $\beta \le 1 \le d$) that cost happens to be dominated by the oracle rate; in the higher-order smoothness scale that local-polynomial extensions would inhabit, it would exceed the oracle rate whenever $\beta > d$, and the validation selector avoids it.
The paired validation difference cancels $\eta^2$ exactly, leaving the noise damped by the candidate gap: its $p$-th central moment shrinks like $|\Delta|^p$ (at $p=2$, Bernstein exponent $1$), enabling oracle-rate selection; and because its expectation is the risk gap itself, no lower bound on regret gaps --- hence no regularity of the noise at zero --- is ever needed.
In mechanism-selection terms: scoring candidate pricing rules by realized gains scores them on rare, lumpy trade events; scoring them by squared distance to the revealed valuations scores them on every round.
Confidence-interval-based selection (Lepski-type) faces a further mechanism-level obstacle, which we record as a heuristic rather than a proven impossibility: a valid observable radius would have to upper-bound a heavy-tailed mean, whose mass is invisible from below without knowledge of $(p, \sigma_p)$.
\end{remark}

Empirically, the validation selector outperforms the gain-based tournament and finishes at or below the tuned oracle width on identical sample paths (Figure~\ref{fig:validation}; protocol in Appendix~\ref{app:experiments}).

\section{Discussion}\label{sec:discussion}

\noindent\textbf{Established:} Lemma~\ref{lem:selfbound} extends the self-bounding property to $\R$-valued valuations under bounded conditional densities and finite first moments, within the conditionally independent, conditionally centered noise model (no variance needed); truncated mean optimality under $p$-th moment~\cite{bubeck2013bandits,catoni2012challenging}; epoch-based upper bounds (parametric and nonparametric); matching lower bounds in $T$ over the effective range $\beta \in (0,1]$ via Assouad + fixed-support mixtures (Propositions~\ref{prop:lower} and~\ref{prop:paramlower}); parameter-free adaptivity at the oracle rate --- parametric (Theorem~\ref{thm:adaptive}) and nonparametric for all $\beta$ (Theorem~\ref{thm:adaptivenp}).
\textbf{Remaining:} optimality of epoch-based approach vs.\ online alternatives; extension to $p \ge 2$ with sub-Gaussian tails.

\smallskip\noindent\textbf{Open questions.}
\textbf{(Q1)}~Can an online robust estimator eliminate the epoch-based $O(\log T)$ overhead?
\textbf{(Q2)}~Under Assumptions~\ref{ass:density} and~\ref{ass:regular} the regret-gap exponent is exactly $2$ (Lemma~\ref{lem:selfbound} and~\eqref{eq:app:unique}); without Assumption~\ref{ass:regular} only the upper direction holds. Can specific tail shapes (e.g., sub-Gaussian) improve the \emph{constant}?
\textbf{(Q3)}~\textit{Heteroskedastic extension:} if $\sigma_p(x)$ depends on context, can regret scale with $\E[\sigma_p(x_t)^2]$ rather than $\max_x \sigma_p(x)^2$?
\textbf{(Q4)}~\textit{Dimension dependence of the lower bound:} the rates are tight in $T$ (Propositions~\ref{prop:lower} and~\ref{prop:paramlower}); whether the parametric upper bound's $d$-dependence is optimal in the conditional model remains open.
\textbf{(Q5)}~\textit{Context-independent noise:}
our lower-bound instances let the conditional noise law vary with the context---the mean
perturbation is carried by mixture weights on a fixed support. Whether the rates of
Theorems~\ref{thm:parametric}--\ref{thm:nonparametric} remain minimax when $f_\xi,f_\zeta$
are a single context-independent pair is open, and we do not expect the present
construction to settle it. For any \emph{fixed} density $f$ with $\E|X|^p\le\sigma_p^p$ and
finite Fisher information, taking $\psi(x)=\operatorname{sgn}(x)|x|^{p-1}$ in the
Cauchy--Schwarz step of Stam's argument gives $I(f)\ge(p-1)^2/\sigma_p^2$. For each such
fixed, sufficiently regular location family this suggests a root-$n$ obstruction to reaching
the moment-matching exponent through a pure location-shift construction: local testing
between shifts at mean distance $\delta$ then needs only $O(\sigma_p^2/\delta^2)$
observations, polynomially fewer than the $\delta^{-p/(p-1)}$ the moment-matching barrier
demands. We stress that this is a heuristic for fixed $f$, not a uniform non-asymptotic
statement over families varying with $T$. A lower bound in that subclass would plausibly
need non-location structure, or the minimax rate there is genuinely faster; we leave the
dichotomy open.

\bibliographystyle{plainnat}
\bibliography{references}

\appendix
\section{Proof of Lemma~\ref{lem:selfbound}}\label{app:selfbound}

We prove all three claims: the self-bounding inequality~\eqref{eq:selfbound}, the unique maximizer property, and the bound $\E[g(m,V,W)] \le 2\sigma_p$.
Throughout we argue conditionally on the context: fix $x$, and write $f_\xi(\cdot) := f_\xi(\cdot \mid x)$, $f_\zeta(\cdot) := f_\zeta(\cdot \mid x)$, $m := m(x)$.
The self-bounding constant $L$ and the gain bound $2\sigma_p$ are uniform in $x$; the uniqueness constants $r_0, c$ below may depend on $x$ (Assumption~\ref{ass:regular} holds for a.e.\ $x$).
Unconditional statements follow by the tower property.

\smallskip\noindent\textbf{Step 1: Noise-coordinate representation.}
Write $\delta = \pi - m$.
Substituting $V = m + \xi$, $W = m + \zeta$:
\begin{equation}\label{eq:app:gnoise}
  g(m{+}\delta,\;m{+}\xi,\;m{+}\zeta)
  = (\xi{-}\zeta)\,\ind\{\zeta \le \delta \le \xi\}
  + (\zeta{-}\xi)\,\ind\{\xi \le \delta \le \zeta\}.
\end{equation}
Define $h(\delta) = \E[g(m{+}\delta,\, V,\, W)]$.
We must show $h(0) - h(\delta) \le L\,|\delta|^2$.

\smallskip\noindent\textbf{Step 2: Decomposition by case.}
Define the one-sided expectations
$\Psi_\xi(\delta) = \E[(\xi - \delta)^+]$ and
$\Phi_\xi(\delta) = \E[(\delta - \xi)^+]$,
and analogously for $\zeta$.
Split $h = I_1 + I_2$ according to the two cases in~\eqref{eq:app:gnoise}:
\begin{align}
  I_1(\delta) &= \iint_{\zeta \le \delta \le \xi} (\xi - \zeta)\,f_\xi(\xi)\,f_\zeta(\zeta)\,d\xi\,d\zeta
  = \int_{-\infty}^{\delta}\!f_\zeta(\zeta)\!\int_{\delta}^{\infty}\!(\xi - \zeta)\,f_\xi(\xi)\,d\xi\,d\zeta, \label{eq:app:I1}\\
  I_2(\delta) &= \iint_{\xi \le \delta \le \zeta} (\zeta - \xi)\,f_\xi(\xi)\,f_\zeta(\zeta)\,d\xi\,d\zeta
  = \int_{\delta}^{\infty}\!f_\zeta(\zeta)\!\int_{-\infty}^{\delta}\!(\zeta - \xi)\,f_\xi(\xi)\,d\xi\,d\zeta. \label{eq:app:I2}
\end{align}

\smallskip\noindent\textbf{Step 3: Differentiation via Leibniz rule.}
We differentiate each integral with respect to $\delta$.
Dominated convergence justifies passing the derivative inside:
the integrand in~\eqref{eq:app:I1} is bounded by $|\xi - \zeta| \le |\xi| + |\zeta|$, which is integrable since $\E[|\xi|] < \infty$ (implied by $p > 1$).

Applying the Leibniz integral rule to~\eqref{eq:app:I1} (differentiating the upper limit of the inner integral and the upper limit of the outer integral):
\begin{equation}\label{eq:app:I1prime}
  I_1'(\delta) = f_\zeta(\delta)\!\int_{\delta}^{\infty}\!(\xi - \delta)\,f_\xi(\xi)\,d\xi
  - f_\xi(\delta)\!\int_{-\infty}^{\delta}\!(\delta - \zeta)\,f_\zeta(\zeta)\,d\zeta
  = f_\zeta(\delta)\,\Psi_\xi(\delta) - f_\xi(\delta)\,\Phi_\zeta(\delta).
\end{equation}
Similarly, from~\eqref{eq:app:I2}:
\begin{equation}\label{eq:app:I2prime}
  I_2'(\delta) = -f_\zeta(\delta)\,\Phi_\xi(\delta) + f_\xi(\delta)\,\Psi_\zeta(\delta).
\end{equation}
Combining:
\begin{equation}\label{eq:app:hprime_raw}
  h'(\delta) = f_\zeta(\delta)\bigl[\Psi_\xi(\delta) - \Phi_\xi(\delta)\bigr]
  + f_\xi(\delta)\bigl[\Psi_\zeta(\delta) - \Phi_\zeta(\delta)\bigr].
\end{equation}

\smallskip\noindent\textbf{Step 4: Simplification to the key formula.}
The pointwise identity $(\xi - \delta)^+ - (\delta - \xi)^+ = \xi - \delta$ yields, upon taking expectations (valid since $\E[|\xi|] < \infty$):
\begin{equation}\label{eq:app:psiphi}
  \Psi_\xi(\delta) - \Phi_\xi(\delta) = \E[\xi] - \delta = -\delta,
\end{equation}
and similarly $\Psi_\zeta(\delta) - \Phi_\zeta(\delta) = -\delta$.
Substituting into~\eqref{eq:app:hprime_raw}:
\begin{equation}\label{eq:app:hprime}
  h'(\delta) = -\delta\bigl[f_\xi(\delta) + f_\zeta(\delta)\bigr].
\end{equation}
\emph{Verification:} $h'(0) = 0$ (consistent with $m$ being a critical point), $h'(\delta) \ge 0$ for $\delta < 0$, and $h'(\delta) \le 0$ for $\delta > 0$ wherever the derivative exists (consistent with $m$ being a maximizer; the inequalities are weak, as the densities may vanish).

Since $f_\xi, f_\zeta$ are merely bounded (not assumed continuous), \eqref{eq:app:I1prime}--\eqref{eq:app:hprime} hold for almost every $\delta$.
Local Lipschitzness of $h$ must be established \emph{directly from the integral representations} (an a.e.-bounded derivative alone would not suffice): changing $\delta$ to $\delta + \eta$ in~\eqref{eq:app:I1} moves only two integration strips of width $|\eta|$ --- the outer variable through $(\delta, \delta+\eta]$ and the inner limit likewise --- and each strip's contribution is at most $L|\eta| \cdot \sup_{s \in [\delta, \delta+\eta]}\E[|\xi| + |\zeta| + |s|] < \infty$; the same holds for~\eqref{eq:app:I2}.
Hence $h = I_1 + I_2$ is locally Lipschitz, therefore absolutely continuous on compacts, its derivative exists a.e.\ and equals~\eqref{eq:app:hprime} wherever the Leibniz computation applies, and the integration in Step~5 is valid.

\smallskip\noindent\textbf{Step 5: Integration and the self-bounding inequality.}
For $\delta > 0$ (the case $\delta < 0$ is symmetric by the substitution $s \mapsto -s$):
\begin{equation}\label{eq:app:integrate}
  h(0) - h(\delta) = -\int_0^{\delta} h'(s)\,ds
  = \int_0^{\delta} s\bigl[f_\xi(s) + f_\zeta(s)\bigr]\,ds.
\end{equation}
Since $f_\xi(s) + f_\zeta(s) \le 2L$ for all $s$:
\begin{equation}\label{eq:app:bound}
  h(0) - h(\delta) \le 2L\int_0^{|\delta|} s\,ds = 2L\cdot\frac{|\delta|^2}{2} = L\,|\delta|^2.
\end{equation}
This is precisely~\eqref{eq:selfbound} since $h(0) - h(\delta) = \E[g(m,V,W) - g(\pi,V,W)]$ and $|\delta| = |\pi - m|$.

\emph{Unique maximizer (under Assumption~\ref{ass:regular}):}
Since $f_\xi$ and $f_\zeta$ are continuous at $0$ with $f_\xi(0) + f_\zeta(0) > 0$, there exist $r_0 > 0$ and $c > 0$ such that
$f_\xi(s) + f_\zeta(s) \ge c$ for all $|s| \le r_0$.
Fix $\delta \ne 0$. The representation~\eqref{eq:app:integrate} (for $\delta > 0$; its symmetric counterpart for $\delta < 0$) has a nonnegative integrand, so dropping the part of the integration range beyond $r_0$,
\begin{equation}\label{eq:app:unique}
  h(0) - h(\delta) \ge \int_0^{\min(|\delta|,\,r_0)} s\,c\,ds
  = \frac{c}{2}\,\min(|\delta|,r_0)^2 > 0.
\end{equation}
Hence $h(\delta) < h(0)$ for every $\delta \ne 0$: the price $m$ is the unique global maximizer.

Assumption~\ref{ass:regular} cannot be dropped: if $f_\xi + f_\zeta \equiv 0$ on $[0, a]$ for some $a > 0$, then~\eqref{eq:app:integrate} gives $h(\delta) = h(0)$ for all $\delta \in [0, a]$, so every price in $[m, m{+}a]$ maximizes the expected gain and uniqueness fails.
The self-bounding inequality~\eqref{eq:app:bound} is unaffected, as it requires only the upper bound $f_\xi, f_\zeta \le L$.

\emph{Gain bound:} $g(m, m{+}\xi, m{+}\zeta) = |\xi - \zeta|\,\ind\{\min(\xi,\zeta) \le 0 \le \max(\xi,\zeta)\} \le |\xi| + |\zeta|$.
Hence $\E[g(m,V,W)] \le \E[|\xi|] + \E[|\zeta|] \le 2\sigma_p$ by Jensen's inequality applied conditionally on $x$ ($\E[|\xi| \mid x] \le (\E[|\xi|^p \mid x])^{1/p} \le \sigma_p$ for $p \ge 1$), followed by the tower property. \qed

\section{A Truncated-Mean Bound with Heterogeneous Means}\label{app:truncated}

This appendix isolates the concentration inequality used in Sections~\ref{sec:proof:param} and~\ref{sec:proof:nonparam}.
It extends Lemma~1 of~\cite{bubeck2013bandits} from i.i.d.\ samples to independent samples whose means may differ, which is what the nonparametric algorithm actually faces: within a cell, the observations $Y_s = m(x_s) + \eta_s$ have means $m(x_s)$ varying from sample to sample.

\begin{lemma}[Truncated mean, heterogeneous means]\label{lem:truncated}
Let $p \in (1,2]$ and let $X_1,\dots,X_n$ be independent real-valued random variables with $\E[|X_s|^p] \le u$ for all $s \in [n]$ and some $u > 0$.
Write $\mu_s = \E[X_s]$ and $\bar\mu = n^{-1}\sum_{s=1}^n \mu_s$.
For $\delta \in (0,1)$ set
\[
  \tau = \Bigl(\frac{u\,n}{\log(1/\delta)}\Bigr)^{\!1/p},
  \qquad
  \hat\mu = \frac{1}{n}\sum_{s=1}^n X_s\,\ind\{|X_s| \le \tau\}.
\]
Then, with probability at least $1 - 2\delta$,
\begin{equation}\label{eq:app:trunc}
  |\hat\mu - \bar\mu| \le 4\,u^{1/p}\Bigl(\frac{\log(1/\delta)}{n}\Bigr)^{\!(p-1)/p}.
\end{equation}
\end{lemma}

\begin{proof}
Write $\widetilde X_s = X_s\,\ind\{|X_s| \le \tau\}$.
Since $\mu_s = \E[\widetilde X_s] + \E[X_s\,\ind\{|X_s| > \tau\}]$,
\begin{equation}\label{eq:app:truncdecomp}
  \hat\mu - \bar\mu
  = \frac{1}{n}\sum_{s=1}^n \bigl(\widetilde X_s - \E[\widetilde X_s]\bigr)
  \;-\; \frac{1}{n}\sum_{s=1}^n \E\bigl[X_s\,\ind\{|X_s| > \tau\}\bigr].
\end{equation}
\emph{Truncation bias.} On $\{|X_s| > \tau\}$ we have $|X_s| \le |X_s|^p\,\tau^{1-p}$, hence for each $s$,
$\bigl|\E[X_s\,\ind\{|X_s| > \tau\}]\bigr| \le \E[|X_s|^p]\,\tau^{1-p} \le u\,\tau^{1-p}$.

\emph{Bernstein term.} Each $\widetilde X_s$ takes values in $[-\tau, \tau]$, so $|\widetilde X_s - \E \widetilde X_s| \le 2\tau$, and
$\operatorname{Var}(\widetilde X_s) \le \E[\widetilde X_s^2] \le \E[|X_s|^p]\,\tau^{2-p} \le u\,\tau^{2-p}$.
Bernstein's inequality for independent bounded random variables gives, with probability at least $1 - 2\delta$,
\[
  \Bigl|\frac{1}{n}\sum_{s=1}^n \bigl(\widetilde X_s - \E[\widetilde X_s]\bigr)\Bigr|
  \le \sqrt{\frac{2\,u\,\tau^{2-p}\log(1/\delta)}{n}} + \frac{4\,\tau\log(1/\delta)}{3n}.
\]
Substituting $\tau = (u\,n/\log(1/\delta))^{1/p}$, each of the three terms (bias, variance, range) equals a constant multiple of $u^{1/p}(\log(1/\delta)/n)^{(p-1)/p}$, with respective constants $1$, $\sqrt{2}$, and $4/3$.
Since $1 + \sqrt{2} + 4/3 < 4$, the claim follows.
\end{proof}

\begin{remark}\label{rem:iidrecover}
For i.i.d.\ samples ($\mu_s \equiv \mu$) Lemma~\ref{lem:truncated} recovers Lemma~1 of~\cite{bubeck2013bandits} up to constants.
The proof uses only independence and the uniform raw-moment bound; the means may differ arbitrarily.
\end{remark}

\begin{corollary}[Conditional per-cell application]\label{cor:conditional}
In the setting of Theorem~\ref{thm:nonparametric}, fix epoch $k$ and condition on the design $X^{(k-1)} = (x_s)_{s \in \mathrm{epoch}\,k-1}$.
For any cell $C_j$ containing $n_j \ge 1$ samples, the observations $\{Y_s = m(x_s) + \eta_s : x_s \in C_j\}$ are conditionally independent with conditional means $m(x_s)$ and $\E[|Y_s|^p \mid X^{(k-1)}] \le (B_m + \sigma_p)^p = \bar u$.
Hence Lemma~\ref{lem:truncated}, applied conditionally with $\delta = 1/(MT)$ and threshold $\tau_j = (\bar u\,n_j/\log(MT))^{1/p}$, yields: with conditional probability at least $1 - 2/(MT)$,
\[
  |\hat{m}_k(c_j) - \bar{m}_j| \le 4\,\bar u^{1/p}\Bigl(\frac{\log(MT)}{n_j}\Bigr)^{\!(p-1)/p},
  \qquad
  \bar{m}_j = \frac{1}{n_j}\sum_{s:\,x_s \in C_j} m(x_s).
\]
Since this failure probability is uniform over designs, and the cell-count event of Step~1 of Section~\ref{sec:proof:nonparam} is $X^{(k-1)}$-measurable, the tower property converts the conditional statement into the unconditional guarantee used there.
\end{corollary}

\begin{proof}
Conditional independence and the conditional means are immediate: the triples $(x_s, \xi_s, \zeta_s)$ are i.i.d., so conditionally on the design the noises are independent with laws determined by the $x_s$, and $\E[Y_s \mid X^{(k-1)}] = m(x_s)$ by conditional centering.
For the moment bound, $|Y_s| \le B_m + |\eta_s|$ pointwise and $\E[|\eta_s|^p \mid X^{(k-1)}] \le \sigma_p^p$ (Assumption~\ref{ass:moment} conditionally on $x_s$), so $\|Y_s\|_p \le B_m + \sigma_p$ conditionally on $X^{(k-1)}$ by Minkowski's inequality.
Apply Lemma~\ref{lem:truncated} conditionally, noting $\bar\mu = \bar{m}_j$, and integrate over the design.
\end{proof}

\section{Experimental Protocol and Reproducibility}\label{app:experiments}

All simulation code is \texttt{numpy}-only (v2.2.6; the figure scripts additionally use \texttt{matplotlib}~3.10.8), with one build target per figure, and will be released publicly upon publication; it is withheld here only to preserve anonymity, and the protocol below specifies every estimator, hyperparameter, seed, and acceptance rule needed to reproduce each result independently.
The full pre-specified run ($12$ settings $\times$ $20$ seeds, $T = 10^6$) takes about $55$ seconds on a laptop, and the $T = 10^7$ horizon extension about three minutes.
Seeds are $0$--$19$ throughout (\texttt{numpy.random.default\_rng}).

\smallskip\noindent\textbf{Exact expected regret.}
For i.i.d.\ symmetric noises with density $f$, Lemma~\ref{lem:selfbound} gives per-round expected regret $2\int_0^{|\delta|} s f(s)\,ds$ at price error $\delta$, with closed forms
\[
  \mathcal{R}_{t(\nu)}(\delta) = \frac{2 c_\nu \nu}{\nu - 1}\Bigl[1 - \bigl(1 + \delta^2/\nu\bigr)^{(1-\nu)/2}\Bigr],
  \qquad c_\nu = \frac{\Gamma((\nu+1)/2)}{\sqrt{\nu\pi}\,\Gamma(\nu/2)},
\]
for Student-$t(\nu)$, and
\[
  \mathcal{R}_{\mathrm{Lomax}(a)}(\delta) = a\Bigl[\frac{(1+|\delta|)^{1-a} - 1}{1-a} + \frac{(1+|\delta|)^{-a} - 1}{a}\Bigr]
\]
for the symmetrized Lomax with density $\tfrac{a}{2}(1+|s|)^{-(a+1)}$.
Both are verified against numerical quadrature at startup.
Similarly, $\sigma_p$ is computed in closed form: $\sigma_p^p = \nu^{p/2}\,\Gamma(\tfrac{p+1}{2})\Gamma(\tfrac{\nu-p}{2})/(\sqrt{\pi}\,\Gamma(\tfrac{\nu}{2}))$ for $t(\nu)$ and $\sigma_p^p = a\,\Gamma(p+1)\Gamma(a-p)/\Gamma(a+1)$ for Lomax$(a)$.

\smallskip\noindent\textbf{Instances.}
Parametric: $d = 5$, $\phi = \mathbb{1}/\sqrt{5}$ ($B = 1$), $x_t$ uniform on the unit sphere (so $\lambda = 1/d$).
Nonparametric: $m(x) = 0.2\,\sin(2\pi x_1)(1 + x_d/2)$ on $[0,1]^d$, $d \in \{1,2\}$ ($L_H \approx 2$, $B_m = 0.3$), $x_t \sim \mathrm{Unif}[0,1]^d$ ($\mu_0 = 1$).
The amplitude keeps the bias scale commensurate with the unit noise scale; with amplitude $1.5$ ($L_H \approx 10$), the width $h$ below yields roughly five cells per axis at $\nu = 1.2$ and the bias term dominates every feasible horizon, so the slope fit reads a transient rather than the rate.

\smallskip\noindent\textbf{Algorithms.}
Epochs $[2^{k-1}, 2^k)$; estimates built from the previous epoch exactly as in Sections~\ref{sec:proof:param}--\ref{sec:proof:nonparam}, run at $p_{\mathrm{alg}} = \nu - 0.05$ with the truncation levels of~\eqref{eq:muhat} (parametric, $u = (B + \sigma_{p_{\mathrm{alg}}})^{p_{\mathrm{alg}}}$) and Section~\ref{sec:proof:nonparam} Step~2 (nonparametric, cell width $h = T^{-(p-1)/(\beta p + d(p-1))}$ at $p = p_{\mathrm{alg}}$, $\beta = 1$).
Baseline: identical epochs and partitions, untruncated least squares / plain cell means.

\smallskip\noindent\textbf{Acceptance protocol (fixed in advance).}
Fitted slope = least squares on $(\log_{10} t, \log_{10} R_t)$ over the last decade $t \in [T/10, T]$, per seed; a setting passes if the median fitted slope is at most the theoretical exponent at $p_{\mathrm{alg}}$ plus $0.05$.
Slopes below the line are compatible with the (worst-case) theory and pass.
Result: $10$ of $12$ pass; the two exceptions and their horizon-extension analysis are discussed in Section~\ref{sec:experiments}.

\smallskip\noindent\textbf{Adaptive experiment (added after the protocol was fixed; Figure~\ref{fig:adaptive}).}
MoM-greedy (Section~\ref{sec:adaptive}; $k_T = \lceil 8\log(dT)\rceil = 124$ blocks at $d = 5$, $T = 10^6$) was run against the oracle-tuned truncation and OLS with \emph{common random numbers}: the same seeds and the same draw order, so all three algorithms see identical sample paths and the curves are pointwise comparable.
Fitted last-decade slopes (median over seeds), against the acceptance exponents of the protocol fixed in advance:

{\small
\begin{center}
\begin{tabular}{@{}lcccc@{}}
\toprule
 & $t(1.2)$ & $t(1.5)$ & $t(1.8)$ & Lomax$(1.5)$ \\
\midrule
acceptance exponent (at $p_{\mathrm{alg}}$) & 0.739 & 0.379 & 0.143 & 0.379 \\
MoM-greedy & 0.592 & 0.361 & 0.284 & 0.483 \\
oracle-tuned truncation & 0.728 & 0.332 & 0.206 & 0.452 \\
OLS & 0.816 & 0.463 & 0.233 & 0.441 \\
\midrule
final regret, MoM / oracle-tuned & 0.051 & 0.234 & 0.608 & 0.288 \\
\bottomrule
\end{tabular}
\end{center}}

\noindent MoM-greedy's slopes pass the acceptance lines at $t(1.2)$ and $t(1.5)$; at $t(1.8)$ and Lomax they sit above the lines at $T = 10^6$ \emph{while the regret level is strictly below} the oracle-tuned curve --- whose slope excesses at these same two cells are the exceptions of Section~\ref{sec:experiments} --- at every last-decade checkpoint.
Extending these two cells to $T = 10^7$ (same protocol as Figure~\ref{fig:horizon}), MoM-greedy's mean-curve local slope likewise declines over the measured windows: $0.249 \to 0.188$ at $t(1.8)$ (against $\theta = 0.143$) and $0.466 \to 0.446$ at Lomax (against $\theta = 0.379$).
We claim level domination, consistent with the rate of Theorem~\ref{thm:adaptive} --- not slope-passing --- for those two cells.

\smallskip\noindent\textbf{Width-selection experiment (added after the protocol was fixed; Figure~\ref{fig:validation}).}
At $d = 1$, the validation selector of Theorem~\ref{thm:adaptivenp} was run against the gain-based tournament of Remark~\ref{rem:whyvalidation} and the tuned oracle width, with common random numbers (identical seeds and draw order across all three algorithms; $T = 10^6$, $20$ seeds).
Final-regret ratios of \emph{mean curves} to the oracle width: validation $0.95$ / $0.67$ at $t(1.5)$ / $t(1.2)$; gain tournament $1.38$ / $0.81$; validation over gain $0.68$ / $0.83$. (Figure~\ref{fig:validation} plots the mean of per-seed ratios instead, whose endpoints are $1.40$ / $0.82$ for the gain tournament and $0.95$ / $0.68$ for validation.)
The validation selector's chosen width exponents concentrate at $4$--$5$ against the theory-optimal $\approx 4.7$ at $t(1.5)$.
In the increment-H\"older scale at $d = 1$ (effectively $\beta \le 1 = d$) the proved validation guarantee and the heuristic gain calibration both cover this regime, so the comparison is one of constants, not rates.

\section{Median-of-Means and the Proofs of Section~\ref{sec:adaptive}}\label{app:adaptive}

The median-of-means device goes back to Nemirovsky and Yudin~\cite{nemirovsky1983problem} and Alon, Matias, and Szegedy~\cite{alon1999space}; see~\cite{bubeck2013bandits,lugosi2019mean} for its heavy-tailed analysis.
We use the following standard form.

\begin{lemma}[Median-of-means deviation]\label{lem:mom}
Let $X_1, \dots, X_n$ be i.i.d.\ with mean $\mu$ and $\E|X_i - \mu|^p \le m_p$ for some $p \in (1,2]$.
For an integer $1 \le \kappa \le n/2$, split the sample into $\kappa$ contiguous blocks of size $b = \lfloor n/\kappa \rfloor$ and let $\hat\mu^{\mathrm{MoM}}$ be the median of the block means.
Then, with probability at least $1 - e^{-\kappa/8}$,
\[
  |\hat\mu^{\mathrm{MoM}} - \mu| \;\le\; (8 m_p)^{1/p}\, b^{-(p-1)/p}
  \;\le\; 16\, m_p^{1/p} \left(\frac{\kappa}{n}\right)^{(p-1)/p}.
\]
Neither $p$ nor $m_p$ enters the estimator; only the block count $\kappa$ does.
\end{lemma}

\begin{proof}
By the von Bahr--Esseen inequality~\cite{vonbahr1965inequalities}, valid for all $p \in [1,2]$, a block sum satisfies $\E|\sum_{i \in \mathrm{block}} (X_i - \mu)|^p \le 2 b\, m_p$, so the block mean satisfies $\E|\bar X_b - \mu|^p \le 2 m_p b^{1-p}$.
By Markov's inequality at $t = (8m_p)^{1/p} b^{-(p-1)/p}$, each block mean deviates by more than $t$ with probability at most $2 m_p b^{1-p} / t^p = 1/4$.
The median deviates by more than $t$ only if at least $\kappa/2$ of the $\kappa$ block means do. The blocks are independent and each deviates with probability \emph{at most} $1/4$, so the number that do is stochastically dominated by $\mathrm{Bin}(\kappa, \tfrac14)$; by Hoeffding's inequality, $\Pr(\mathrm{Bin}(\kappa, \tfrac14) \ge \tfrac{\kappa}{2}) \le e^{-2\kappa(1/4)^2} = e^{-\kappa/8}$.
Finally $b \ge n/(2\kappa)$ for $\kappa \le n/2$, and $8^{1/p} 2^{(p-1)/p} \le 16$.
\end{proof}

\subsection{Proof of Theorem~\ref{thm:adaptive}}

Fix any $p \in (1,2]$, $\sigma_p$, $B$, $\lambda$ satisfied by the instance; none is used by the algorithm.

\smallskip\noindent\emph{Step 1 (per-coordinate estimation).}
As in Section~\ref{sec:proof:param}, the scores $\{[S_s]_j\}_{s=1}^n$ are i.i.d.\ with mean $[\Sigma\phi]_j$; their central moment satisfies
\[
  \E\bigl|[S_s]_j - [\Sigma\phi]_j\bigr|^p
  \le 2^{p-1}\bigl(\E|[S_s]_j|^p + |[\Sigma\phi]_j|^p\bigr)
  \le 2^{p-1}\bigl((B+\sigma_p)^p + B^p\bigr) \le 2^p (B+\sigma_p)^p =: m_p,
\]
using~\eqref{eq:rawmoment} and $|[\Sigma\phi]_j| \le \|\Sigma\|_{\mathrm{op}}\|\phi\| \le B$ (as $\|\Sigma\|_{\mathrm{op}} \le \E\|x_t\|^2 \le 1$).
Apply Lemma~\ref{lem:mom} with $\kappa = k_T = \lceil 8\log(dT)\rceil$: for each $j$, with probability $\ge 1 - e^{-k_T/8} \ge 1 - 1/(dT)$,
\begin{equation}\label{eq:mompercoord}
  |\hat\mu_j - [\Sigma\phi]_j| \le 32\,(B+\sigma_p)\left(\frac{k_T}{n}\right)^{(p-1)/p},
  \qquad k_T = \lceil 8\log(dT) \rceil,
\end{equation}
and a union bound over $j \in [d]$ gives $\|\hat\mu - \Sigma\phi\|_\infty$ bounded as in~\eqref{eq:muconc} with the constant $4$ replaced by $32$ and $\log(dT)$ by $k_T \le 8\log(dT) + 1$ (absorbed into $\Otil$), with probability $\ge 1 - 1/T$.

\smallskip\noindent\emph{Step 2 (Gram, prediction, epochs) --- unchanged.}
The matrix Bernstein step, the non-asymptotic domination argument, and the per-epoch/geometric-sum assembly of Section~\ref{sec:proof:param} contain no $p$-dependent algorithmic quantity and go through verbatim with~\eqref{eq:percoord} replaced by~\eqref{eq:mompercoord}; the good event has probability $\ge 1 - 2/T$ per epoch.

\smallskip\noindent\emph{Step 3 (everything else is capped).}
Since $g \ge 0$, the per-round regret of any price is at most $\E[g(m(x_t), V_t, W_t)] \le 2\sigma_p$ (Lemma~\ref{lem:selfbound}).
This cap absorbs: (i) bad events, as in~\eqref{eq:badepoch}; (ii) epochs with $n < 2k_T$, of which there are $O(\log\log(dT))$ containing $O(\log(dT))$ rounds in total; and (iii) any round where $\hat\Sigma$ is ill-conditioned and $\hat\phi_k$ is wild --- no projection or boundedness of the price is ever needed.
Summing exactly as in Section~\ref{sec:proof:param} yields $R_T = \Otil(L\,d\,(B+\sigma_p)^2\lambda^{-2}\,T^{(2-p)/p})$. \qed

\subsection{The cell-width tournament: algorithm}\label{app:adaptive:alg}

Candidates are cell widths $h_i = 2^{-i}$, $i \in \{0, 1, \dots, \lceil \log_2 T\rceil\}$; write $N \le \log_2 T + 2$ for their number.
Candidate $i$ maintains the epoch-partition scheme of Section~\ref{sec:proof:nonparam} at width $h_i$, with per-cell \emph{median-of-means} (block count $k_T' := \lceil 8\log(N^2 M_{\max} T)\rceil$, where $M_{\max} := \max_i \lceil h_i^{-1}\rceil^d \le (2T)^d$ is the largest cell count on the grid, so $k_T' = O(d \log T)$; Lemma~\ref{lem:hetmom} below) in place of truncated means, so every candidate is parameter-free.
Within a candidate, cells of its current partition holding fewer than $2k_T'$ samples from the estimation epoch --- in particular, empty cells --- are priced at $0$, as in Section~\ref{sec:proof:nonparam}; these are the \emph{burn-in cells} of the envelope argument below (pointwise error at most $B_m$).
Let $\pi_i^{(k)}$ denote candidate $i$'s price function built from epoch-$(k{-}1)$ data; $\pi_i^{(k)}$ is frozen throughout epoch $k$.

\emph{Certify-then-play schedule.}
During epoch $k$ the broker plays $\pi_{\hat c_{k-1}}^{(k-1)}$ --- the price function of the candidate $\hat c_{k-1}$ selected at the end of epoch $k{-}1$, \emph{in the version that was certified there} (one epoch stale; for the covered candidate this costs a factor $2^{2(p-1)/p} \le 2$ in the per-round bound, since its estimate uses half the samples).
At the end of epoch $k$, the broker computes every candidate's validation losses $v_t(i) = (\hat m_{h_i}^{(k)}(x_t) - Y_t)^2$ on the epoch's rounds, runs the max--min tournament of Lemma~\ref{lem:tournament} (a median-of-means tournament in the sense of Lugosi and Mendelson~\cite{lugosi2016tournaments}) on the median-of-means compared pairwise differences (Lemma~\ref{lem:compare}), and selects $\hat c_k$; epochs with $n_k < 2k_T'$ skip the tournament and retain the previous selection (initially the coarsest width $h_0 = 1$).
The algorithm's inputs are $d$ and $T$ alone.

\subsection{Supporting lemmas}

\begin{lemma}[Median-of-means, independent heterogeneous means]\label{lem:hetmom}
Let $X_1, \dots, X_n$ be independent with means $\mu_s \in [\bar\mu - w, \bar\mu + w]$ and central moments $\E|X_s - \mu_s|^p \le m_p$, $p \in (1,2]$.
For $1 \le \kappa \le n/2$ blocks: with probability at least $1 - e^{-\kappa/8}$,
\[
  |\hat\mu^{\mathrm{MoM}} - \bar\mu| \;\le\; w + 16\, m_p^{1/p} (\kappa/n)^{(p-1)/p}.
\]
\end{lemma}

\begin{proof}
Each block mean decomposes as (block average of $\mu_s$) $+$ (block average of the centered parts); the deterministic part lies within $w$ of $\bar\mu$, and the centered part obeys the von Bahr--Esseen/Markov step of Lemma~\ref{lem:mom} verbatim (von Bahr--Esseen requires independence only, not identical distribution).
Hoeffding's median boost is unchanged.
\end{proof}

\noindent Applied per cell with $w = L_H (\sqrt{d}\,h)^\beta$ (the within-cell mean spread) and $m_p \le \sigma_p^p$ (conditional noise moments, conditionally on the design as in Corollary~\ref{cor:conditional}), Lemma~\ref{lem:hetmom} replaces the truncated-mean step of Section~\ref{sec:proof:nonparam}: every candidate $i$ satisfies the analogue of~\eqref{eq:nppred2} at its own width $h_i$, with constants $C_{\mathrm{est}}' = C(p)\,\sigma_p (2^{d+1}/\mu_0)^{(p-1)/p}$ and no knowledge of $(p, \sigma_p, B_m, \beta)$.

\begin{lemma}[Paired-difference structure]\label{lem:paired}
Let $\pi, \pi'$ be measurable price functions, $\Delta(x) = |\pi(x) - \pi'(x)|$, $\bar\rho(x) = \max(|\pi(x) - m(x)|, |\pi'(x) - m(x)|)$, $Z = |\xi - \zeta|$, and $D = g(\pi(x), V, W) - g(\pi'(x), V, W)$.
Then, conditionally on $x$:
\begin{enumerate}[label=(\roman*)]
\item $D = \pm Z\,\ind\{S\}$, where $S$ is the event that exactly one of the two prices lies in $[V \wedge W, V \vee W]$;
\item $\Pr(S \mid x) \le 4L\,\Delta(x)$;
\item $\E[Z^p \ind\{S\} \mid x] \le 2^{p+2} L\, \Delta(x)\,\bigl(\sigma_p^p + \bar\rho(x)^p\bigr)$, and unconditionally on the straddle structure, $\E[|D|^p \mid x] \le \E[Z^p \mid x] \le 2^p \sigma_p^p$;
\item the conditional expected gap $\gamma(x) := \E[D \mid x]$ satisfies $|\gamma(x)| \le 2L\,\bar\rho(x)\,\Delta(x)$.
\end{enumerate}
\end{lemma}

\begin{proof}
(i) The gain takes values in $\{0, Z\}$ --- $Z$ iff the price lies in the value interval --- so the difference vanishes unless exactly one price is inside.
(ii) $S$ forces an endpoint of the value interval into the price gap $J = [\pi \wedge \pi', \pi \vee \pi']$; the conditional densities of $V \wedge W$ and $V \vee W$ are bounded by $f_V + f_W \le 2L$, and there are two endpoints.
(iii) On the min-endpoint event, $|V \wedge W - m| \le \bar\rho$, so $Z \le \max(|\xi|, |\zeta|) + \bar\rho$ and $Z^p \le 2^{p-1}(\max(|\xi|,|\zeta|)^p + \bar\rho^p)$; splitting on which noise realizes the minimum and localizing that noise to $J - m$ (an interval of length $\Delta$ on which $|u| \le \bar\rho$, carrying density at most $L$) gives $\E[\max(|\xi|,|\zeta|)^p \ind\{\cdot\}] \le 2L\Delta(\bar\rho^p + \sigma_p^p)$; symmetrizing over the two orderings and the two endpoints yields the constant $2^{p+2}$.
The global bound is $|D| \le Z$ pointwise plus $\E[|\xi - \zeta|^p \mid x] \le 2^{p-1}(\E|\xi|^p + \E|\zeta|^p) \le 2^p\sigma_p^p$.
(iv) $\gamma(x) = h_x(\pi - m) - h_x(\pi' - m)$ and $|h_x'(s)| = |s|[f_\xi(s|x) + f_\zeta(s|x)] \le 2L|s|$ a.e.; integrating $h_x'$ over the interval between the two price errors (valid since $h_x$ is locally Lipschitz, hence absolutely continuous --- the argument of Appendix~\ref{app:selfbound}, Step~4) gives the claim.
\end{proof}

\noindent Lemma~\ref{lem:paired} is used in this paper only for Remark~\ref{rem:whyvalidation}'s selection-cost analysis; the selection itself runs on validation losses, whose comparison statistics are controlled next.

\begin{lemma}[Validation comparison]\label{lem:compare}
Condition on the history before epoch $k$, so all candidate functions $\hat m_{h_i}^{(k)}$ are fixed and the epoch-$k$ triples are i.i.d.; write $\mathrm{risk}(i) := \E_x[(\hat m_{h_i}^{(k)}(x) - m(x))^2]$ and $\Delta_{i,i'}(x) := \hat m_{h_i}^{(k)}(x) - \hat m_{h_{i'}}^{(k)}(x)$.
For each ordered pair $(i, i')$, let $\widehat D(i, i')$ be the median-of-means, over $\kappa = k_T'$ blocks of the $n_k$ rounds, of the per-round validation differences $v_t(i') - v_t(i)$.
Then
\[
  \E[v_t(i') - v_t(i)] = \mathrm{risk}(i') - \mathrm{risk}(i) =: \Gamma(i, i') \qquad \text{exactly},
\]
and on the event that every candidate's function satisfies its own uniform prediction bound (so that $|\hat m_{h_i}^{(k)} - m| \le \mathcal{E} := C_3(B_m + \sigma_p)$ pointwise for all $i$; burn-in cells price $0$ and contribute at most $B_m$), with probability at least $1 - N^2 e^{-\kappa/8} \ge 1 - 1/T$ simultaneously for all pairs:
\[
  |\widehat D(i,i') - \Gamma(i,i')| \;\le\; r(i,i') := \bar\sigma\,\Bigl(\frac{\kappa}{n_k}\Bigr)^{\!(p-1)/p} \|\Delta_{i,i'}\|_2,
  \qquad \bar\sigma := C_1(p)\,(\mathcal{E} + \sigma_p),
\]
with $\|\Delta_{i,i'}\|_2^2 \le 2\,(\mathrm{risk}(i) + \mathrm{risk}(i'))$.
\end{lemma}

\begin{proof}
The identity: $v_t(i') - v_t(i) = [(\hat m_{h_{i'}} - m)^2 - (\hat m_{h_i} - m)^2](x_t) + 2\eta_t \Delta_{i,i'}(x_t)$ --- the $\eta^2$ terms cancel algebraically --- and $\E[\eta_t \Delta(x_t)] = \E[\Delta(x_t)\E[\eta_t \mid x_t]] = 0$ by conditional centering, leaving the risk gap.
The central moment: the deterministic-given-$x$ part is bounded pointwise by $|\Delta|\cdot(|\hat m_{h_i} - m| + |\hat m_{h_{i'}} - m|) \le 2\mathcal{E}|\Delta(x)|$, and the noise part has conditional moment $\E[|2\eta\Delta|^p \mid x] \le 2^{2p}\sigma_p^p|\Delta(x)|^p$; collecting, $\E|v\text{-diff} - \Gamma|^p \le C(p)(\mathcal{E} + \sigma_p)^p\,\E|\Delta|^p \le C(p)(\mathcal{E}+\sigma_p)^p \|\Delta\|_2^p$ by Jensen ($p \le 2$).
Apply Lemma~\ref{lem:mom} and union over the $\le N^2$ ordered pairs with $\kappa = k_T' = \lceil 8\log(N^2 M_{\max} T)\rceil$.
Finally $\|\Delta\|_2^2 \le 2\|\hat m_{h_i} - m\|_2^2 + 2\|\hat m_{h_{i'}} - m\|_2^2$.
\end{proof}

\begin{lemma}[Max--min selection]\label{lem:tournament}
Let $\rho : \{1..N\} \to [0,\infty)$ be any functional, $\Gamma(i,i') := \rho(i') - \rho(i)$, and suppose the statistics $\widehat D(i,i')$ satisfy $|\widehat D(i,i') - \Gamma(i,i')| \le r(i,i')$ for all pairs, on some event.
Then $\hat c = \arg\max_i \min_{i'} \widehat D(i,i')$ satisfies, on that event and for any candidate $c^*$:
\[
  \rho(\hat c) \;\le\; \rho(c^*) \;+\; r(\hat c, c^*) \;+\; \max_{i'} \bigl[\, r(c^*, i') - \Gamma(c^*, i')\,\bigr]_+ .
\]
Consequently, applied with $\rho = \mathrm{risk}$ and the radii of Lemma~\ref{lem:compare}, with $\omega := (\kappa/n_k)^{(p-1)/p}$:
\[
  \mathrm{risk}(\hat c) \;\le\; 5\,\mathrm{risk}(c^*) \;+\; 4\,\bar\sigma^2\,\omega^2 .
\]
\end{lemma}

\begin{proof}
\emph{First display.} On the event, $\min_{i'} \widehat D(c^*, i') \ge \min_{i'} [\Gamma(c^*, i') - r(c^*, i')] \ge -\max_{i'}[r(c^*,i') - \Gamma(c^*,i')]_+$ (the $[\cdot]_+$ absorbs the sign of $\Gamma$, which need not be nonnegative for arbitrary $c^*$).
By maximality, $\widehat D(\hat c, c^*) \ge \min_{i'} \widehat D(\hat c, i') \ge \min_{i'} \widehat D(c^*, i')$, and $\Gamma(\hat c, c^*) \ge \widehat D(\hat c, c^*) - r(\hat c, c^*)$; combining and rearranging $\Gamma(\hat c, c^*) = \rho(c^*) - \rho(\hat c)$ gives the display.

\emph{Second display (two applications of Young's inequality).}
By Lemma~\ref{lem:compare}, $r(c^*, i') \le \bar\sigma\omega\,\|\Delta\|_2 \le \bar\sigma\omega\sqrt{2\,(\mathrm{risk}(i') + \mathrm{risk}(c^*))}$.
\emph{The bracket:} with $u := \sqrt{\mathrm{risk}(i') + \mathrm{risk}(c^*)}$ and $\Gamma(c^*,i') = \mathrm{risk}(i') - \mathrm{risk}(c^*) = u^2 - 2\,\mathrm{risk}(c^*)$,
\[
  \bigl[\, r(c^*,i') - \Gamma(c^*,i')\,\bigr]_+
  \;\le\; \bigl[\, \sqrt{2}\,\bar\sigma\omega\,u - u^2 \,\bigr]_+ + 2\,\mathrm{risk}(c^*)
  \;\le\; \tfrac12\,\bar\sigma^2\omega^2 + 2\,\mathrm{risk}(c^*),
\]
maximizing the concave quadratic in $u$ at $u_* = \bar\sigma\omega/\sqrt2$.
\emph{The self-referential term:} by Young ($ab \le b^2/4 + a^2$),
$r(\hat c, c^*) \le \sqrt2\,\bar\sigma\omega\sqrt{\mathrm{risk}(\hat c) + \mathrm{risk}(c^*)} \le \tfrac14\bigl(\mathrm{risk}(\hat c) + \mathrm{risk}(c^*)\bigr) + 2\bar\sigma^2\omega^2$.
Substituting both into the first display:
$\mathrm{risk}(\hat c) \le \mathrm{risk}(c^*) + \tfrac14\mathrm{risk}(\hat c) + \tfrac14\mathrm{risk}(c^*) + 2\bar\sigma^2\omega^2 + \tfrac12\bar\sigma^2\omega^2 + 2\,\mathrm{risk}(c^*)$,
and rearranging, $\tfrac34\,\mathrm{risk}(\hat c) \le \tfrac{13}{4}\,\mathrm{risk}(c^*) + \tfrac52\,\bar\sigma^2\omega^2$, i.e.\ $\mathrm{risk}(\hat c) \le 5\,\mathrm{risk}(c^*) + 4\,\bar\sigma^2\omega^2$.
\end{proof}

\subsection{Proof of Theorem~\ref{thm:adaptivenp}}

Fix the instance parameters $(p, \beta, L_H, \sigma_p, B_m)$; none is known to the algorithm.
Let $c^*$ be the grid width within a factor $2$ of the oracle width~\eqref{eq:hopt}; its constants inflate by $2^{O(\beta + d)}$.

\emph{Covered candidate.}
By Lemma~\ref{lem:hetmom} applied per cell (conditionally on the design, as in Corollary~\ref{cor:conditional}) and the cell-count argument of Section~\ref{sec:proof:nonparam} Step~1, candidate $c^*$'s epoch-$k$ function --- built from the $n_{k-1}$ samples of epoch $k{-}1$ --- satisfies the analogue of~\eqref{eq:nppred2}: per-round expected regret at most $L\,\varepsilon_k^*{}^2$ with $\varepsilon_k^* \le C_{\mathrm{est}}'\,(\kappa\, h_*^{-d}/n_{k-1})^{(p-1)/p} + 2L_H (\sqrt d\, h_*)^\beta$.

\emph{Envelope.}
On the union of the candidates' good events, every candidate's function is uniformly controlled: burn-in cells price $0$ (pointwise error $\le B_m$), and for each filled cell, applying Lemma~\ref{lem:hetmom} \emph{with $\bar\mu = 0$ and $w = B_m$} (the cell samples have means $m(x_s) \in [-B_m, B_m]$) gives $|\hat m_{h_i}^{(k)}| \le B_m + 16\sigma_p(\kappa/n_j)^{(p-1)/p} \le B_m + 16\sigma_p$, hence $|\hat m_{h_i}^{(k)} - m| \le \mathcal{E} := 2B_m + 16\sigma_p$ pointwise, simultaneously for all $i$ --- the envelope required by Lemma~\ref{lem:compare}, uniform in $(L_H, \beta)$.
Moreover $\mathrm{risk}(c^*) \le \varepsilon_k^{*\,2}$.

\emph{Selection.}
By Lemmas~\ref{lem:compare}--\ref{lem:tournament}, on the per-epoch uniform event (probability $\ge 1 - 2/T$), the certified candidate satisfies
$\mathrm{risk}(\hat c_k) \le 5\,\mathrm{risk}(c^*) + 4\bar\sigma^2(\kappa/n_k)^{2(p-1)/p}$, and hence, by the self-bounding upper direction alone (Lemma~\ref{lem:selfbound}),
$\mathrm{Reg}(\hat c_k) \le L\,\mathrm{risk}(\hat c_k) \le 5L\,\varepsilon_k^{*\,2} + 4L\bar\sigma^2(\kappa/n_k)^{2(p-1)/p}$; off the event, the per-round cap $2\sigma_p$ applies.
No lower bound on regret gaps is used anywhere, which is why no regularity of the noise at zero is assumed.
The certify-then-play schedule makes epoch $k{+}1$'s played function exactly the certified one, at the cost of one epoch of staleness for $c^*$ (a factor $\le 2$).

\emph{Assembly.}
Summing over doubling epochs, the covered candidate's part reproduces the oracle rate of Theorem~\ref{thm:nonparametric} (the geometric sums of Section~\ref{sec:proof:nonparam} Step~5, with the new constants), and the selection part contributes
\[
  \sum_k n_k \cdot 4L\bar\sigma^2\Bigl(\frac{\kappa}{n_{k-1}}\Bigr)^{\!2(p-1)/p} = \Otil\bigl(T^{\,(2-p)/p}\bigr),
\]
dominated by the last epoch; and $T^{(2-p)/p}$ is dominated by the oracle rate for \emph{every} $\beta$, since the exponent comparison
$\tfrac{2(p-1)}{p} \ge \tfrac{2\beta(p-1)}{\beta p + d(p-1)} \iff \beta p + d(p-1) \ge \beta p \iff d(p-1) \ge 0$
holds always.
Burn-in accounting: \emph{no additional selection burn-in arises}; epochs with $n_k < 2k_T'$ contribute polylog terms through the cap, and the covered candidate's cell-fill phase is the $O(h_*^{-d}\,\mathrm{polylog}\,T)$ accounting of Section~\ref{sec:proof:nonparam} Step~1 --- $T$-polynomial but dominated by the oracle rate, exactly as there. \qed

\end{document}